\documentclass[letterpaper, 10pt]{article}

\usepackage[margin = 1.25in]{geometry}

\usepackage[T1]{fontenc}
\usepackage[utf8]{inputenc}
\usepackage{lmodern}

\usepackage{url}            %
\usepackage{booktabs}       %
\usepackage{nicefrac}       %
\usepackage{microtype}      %

\usepackage{amsmath,amssymb,amsfonts,textcomp,graphicx,nicefrac,mathtools,mathrsfs, dsfont, amsthm, enumitem} 
\usepackage{dsfont}

\usepackage{bbm}

\usepackage[title]{appendix}

\usepackage{soul}

\usepackage{mathtools}
\usepackage{footmisc}
\usepackage{tikz}

\usepackage{xcolor}
\usepackage{float}

\usepackage{subcaption}

\usepackage[colorlinks=true, linkcolor=red, citecolor=blue]{hyperref}
\usepackage{multirow}
\makeatletter
\def\thm@space@setup{\thm@preskip=2pt
\thm@postskip=0pt}
\makeatother

\newtheoremstyle{dotless}{}{}{\itshape}{}{\bfseries}{}{ }{}
\theoremstyle{dotless}

\newtheorem{def_num}{Definition}

\newtheorem*{defi}{Definition}
\theoremstyle{plain}
\newtheorem{myth}{Theorem}

\newtheorem{mylem}[myth]{Lemma}

\newtheoremstyle{named}{}{}{\itshape}{}{\bfseries}{.}{.5em}{#1 #3}
\theoremstyle{named}
\newtheorem*{namthm*}{Theorem}



\newtheorem{innercustomgeneric}{\customgenericname}
\providecommand{\customgenericname}{}
\newcommand{\newcustomtheorem}[2]{%
  \newenvironment{#1}[1]
  {%
   \renewcommand\customgenericname{#2}%
   \renewcommand\theinnercustomgeneric{##1}%
   \innercustomgeneric
  }
  {\endinnercustomgeneric}
}

\newcustomtheorem{customthm}{Theorem}
\newcustomtheorem{customlemma}{Lemma}

\newcommand{\ind}[1]{\mathds{1}\left\{ #1 \right\}}

\newcommand{\vc}{\textsc{vc}}
\newcommand{\bud}{\mathsf{B}}
\usepackage{algorithm}
\usepackage[noend]{algpseudocode}

\makeatletter
\def\BState{\State\hskip-\ALG@thistlm}
\makeatother
\DeclareMathOperator*{\argmin}{argmin}

\usepackage[
backend=bibtex,
maxalphanames =4,
useprefix=true,
style=alphabetic,%
sorting=anyt,
maxbibnames=99
]{biblatex}
\title{Budget Learning via Bracketing}

\author{ {Aditya Gangrade, Durmus Alp Emre Acar, Venkatesh Saligrama} \\ \large{Boston University} \\ \normalsize{\texttt{\{gangrade, alpacar, srv\}@bu.edu}} 
}
\date{\vspace{-12pt}}
\begin{document}

\maketitle %

\setlength\parskip{4pt plus 1pt minus 1pt}
\setlength\parindent{0pt}

\begin{abstract}

Conventional machine learning applications in the mobile/IoT setting transmit data to a cloud-server for predictions. Due to cost considerations (power, latency, monetary), it is desirable to minimise device-to-server transmissions. The budget learning (BL) problem poses the learner's goal as minimising use of the cloud while suffering no discernible loss in accuracy, under the constraint that the methods employed be edge-implementable. 

We propose a new formulation for the BL problem via the concept of bracketings. Concretely, we propose to sandwich the cloud's prediction, $g,$ via functions $h^-, h^+$ from a `simple' class so that $h^- \le g \le h^+$ nearly always. On an instance $x$, if $h^+(x)=h^-(x)$, we leverage local processing, and bypass the cloud. We explore theoretical aspects of this formulation, providing PAC-style learnability definitions; associating the notion of budget learnability to approximability via brackets; and giving VC-theoretic analyses of their properties. We empirically validate our theory on real-world datasets, demonstrating improved performance over prior gating based methods.
\end{abstract}

\section{Introduction}

Edge devices in mobile and IoT applications are battery and processing power limited. This imposes severe constraints on the methods implementable in such settings - for instance, the typical CPU-based structure of such devices precludes the use of many convolutional layers in vision tasks due to computational latency \cite{zhou2019distributing}, imposing architectural constraints. In particular, modern high accuracy methods like deep neural networks are seldom implementable in these settings. At the same time, edge devices are required to give fast and accurate decisions. Enabling such mechanisms is an important technical challenge.%

Typically, practitioners either learn weak models that can be implemented on the edge (e.g.~\cite{wu2019fbnet, pmlr-v70-kumar17a, hinton2015distilling}), which suffer more errors, or they learn a complex model, which is implemented in a cloud\footnote{or, more realistically, purchase access to a cloud-based model owned by a company that has sufficient data and computational power (e.g.~\cite{mlkit, coreml}).}. The latter solution is also not ideal - cloud access must be purchased, the prediction pipeline suffers from communication latency, and, since communication consumes the majority of the battery power of such a device \cite{zhu_iot}, such solutions limit the device's operational lifetime (see also industry articles, e.g.~\cite{norman_2019, hollemans_2017}). A third option, largely unexplored in practice, is a hybrid of these strategies - we may learn mechanisms to filter out `easy' instances, which may be classified at the edge, and send `difficult' instances to the cloud. The reduction in cloud usage provides direct benefits in,  e.g., battery life, yet accuracy may be retained. Similar concerns apply in many contexts, e.g.~in medicine, security, and web-search \cite{JMLR:v15:xu14a,nan2017adaptive}. %

The key challenge in these applications is to maintain a high accuracy while keeping the usage of the complex classifier, i.e.~the budget, low. To keep accuracy high, we enforce that on the locally predicted instances, the prediction \emph{nearly always agrees} with the cloud. This is thus a problem of `bottom-up' budget learning (BL).%

The natural approach to BL is via the `gating formulation': one learns a gating function $\gamma$, and a local predictor $\pi$, such that if $\gamma = 1$ then $\pi$ is queried, and otherwise the cloud is queried. Unfortunately, this setup is computationally difficult, since the overall classifier involves the product $\pi \cdot \gamma,$ and optimising over the induced non-covexity is hard. Previous efforts try to meet this head on, but either yield inefficient methods, or require difficult to justify relaxations.%

\subsubsection*{Our Contributions}

Our main contribution is a novel formulation of the BL problem, via the notion of brackets, that sidesteps this issue. For functions $h^- \le h^+,$ the bracket $[h^-, h^+]$ $:= \{ f: h^- \le f\le h^+\}$. Brackets provide accurate pointwise control on a binary function - for $f \in [h^-, h^+]$, if $h^+(x) = h^-(x),$ then $f(x)$ takes the same value. We propose to learn a bracketing of the cloud, predicting locally when this condition holds.%

The key advantage of this method arises from the surprising property that we may learn optimal brackets via two \emph{decoupled} learning problems - separately approximating the function from above and from below. These one-sided problems are tractable under convex surrogates, with minimal statistical compromises. Further, this comes at negligible loss of expressivity compared to gating - the existence of good gates and predictors implies the existence of equally good brackets.%

Since expressivity is retained, bracketings lead naturally to definitions of learnability that are theoretically analysable. We define a PAC-style approach to one-sided learning, and provide a VC-theoretic characterisation of the same. We also identify the key budget learning problem as an approximation theoretic question - \emph{which complex classes have `good' bracketings by simple classes?} We characterise this for a binary version of H\"{o}lder smooth classes, and also provide partial results for generic classes with bounded VC dimension.%

Finally, to validate the formulation, we implement the bracketing framework on a binary versions of MNIST and CIFAR classification tasks. With a strong disparity in the cloud and edge models (\S\ref{sec:exp_main}), we obtain usages of $20-40\%$ at accuracies higher than $98\%$ with respect to the cloud. Further, we outperform existing methods in usage by factors of $1.2-1.4$ at these high accuracies.%

\subsubsection*{Related Work}

A common approach is to simply learn \emph{local classifiers with no cloud usage}. If the cloud model is available, one can use methods such as distillation \cite{hinton2015distilling}, and in general one can train classifiers in a resource aware way (e.g. \cite{pmlr-v70-kumar17a, GuptaSGSPKGUVJ17, wu2019fbnet}). The main limitation of this approach is that if the setting is complex enough for a cloud to be needed, then in general such methods cannot attain a similar accuracy level.

\emph{Top-Down and Sequential Approaches} are based on successively learning classifiers of increasing complexity, incorporating the previously learned classifiers (see \cite{JMLR:v15:xu14a,pmlr-v31-trapeznikov13a, NIPS2015_5982, nan2016pruning, bolukbasi2017adaptive}). This approach suffers a combinatorial explosion in the complexity of the learning problems. Recent efforts utilise reinforcement learning methods to rectify this (e.g.~\cite{janisch2019classification, janisch2019classification2, peng2018refuel}).

The BL problem is intimately related to \emph{learning with abstention} (LwA). Indeed, sending an example to the cloud is the same as abstaining on it. The twist in BL is twofold - we assume that a noiseless ground truth, i.e., the `cloud classifier' exists, while LwA tends to concentrate on settings where the labels are noisy; and the class of locally implementable models is much weaker than the class known to contain the cloud model, while the LwA literature is generally not concerned with `simple' classifiers. In addition, no theoretical work on LwA captures this setting. {Perhaps the closest is the study of `perfect selective classification' in \cite{el2010foundations, wiener2011agnostic}, but this work focuses on the stringent condition of getting perfect agreement with certainty, and only gives analyses for classes with controlled disagreement coefficients.}

\emph{Plug-in methods} utilise a pre-trained low complexity model, and learn a gate by estimating its low-confidence regions. We note that much of the theoretical analysis for LwA concentrates on such methods, e.g.~\cite{herbei2006classification, bartlett2008classification, denis2019consistency, shekhar2019binary}.\footnote{\cite{herbei2006classification} also analyse ERM in the setting where a fixed cost for abstention is available.} The principal disadvantage here is that these classifiers are not tuned to the BL problem. However, even crude methods such as gating by thresholding the softmax response of a classifier are very effective (see \S\ref{sec:exp_main}), and serve as strong baselines as observed in \cite{geifman2017selective, geifman2019selectivenet} in the setting of deep neural networks.%

A number of methods aim at \emph{jointly learning gating and prediction} functions (c.f.~\S\ref{sec:gating}). Some of these belong to the LwA literature - \cite{geifman2019selectivenet} proposes to ignore the non-convexity, and use SGD to optimise a loss of the form $\widehat{\mu}( \pi \neq g|\gamma = 1)$ subject to a budget constraint, while \cite{cortes2016learning} instead proposes the relaxation $\pi \gamma \le (\pi + \gamma)/2,$ and optimise this upper bound via convex relaxations. In the BL literature, \cite{nan2017adaptive, nan2017dynamic} propose to relax the problem by introducing an auxiliary variable to decouple $\pi$ and $\gamma$, and then perform alternating minimisation with a KL penalty between the gate and the auxiliary. Note that while each of these papers further specifies algorithms to train classifiers, their main conceptual contribution is the method they take to ameliorate the essential non-convexity of the gating setup. In contrast, our new formulation sidesteps this issue entirely.

Our approach to one-sided learning is related to \emph{Neyman-Pearson classification} \cite{cannon2002learning, scott2005neyman}, with the difference that instead of studying the conditional risks, we are concerned with restricting the total risk subject to one-sided constraints. This leads to the generalisation errors of one-sided learning scaling with the total sample size, as opposed to the per-class sample sizes (see \S\ref{sec:os_learnability}).%

\emph{Bracketings} are important in empirical process theory - for instance, `bracketable' classes characterise the universal Glivenko-Cantelli property \cite{vanHandel2013_arxiv}. While there are generic estimates of the bracketing entropies of various function classes (e.g.~Ch2 of \cite{vaartWellner_arxiv}), these typically do not constrain for complexity of the resulting brackets, and thus their application in our setting is limited. Instead, we explicitly aim to bracket functions by \emph{simple} function classes (see \S\ref{sec:bud_learnability_bones}). We note, however, that our results towards this are preliminary.

\section{Definitions and Formulations} \label{sec:def}

We will restrict discussion to binary functions on the domain $\mathcal{X}$, which is assumed to be compact\footnote{Issues of measurability, and of existence of minimisers of optimisation problems posed as infima are suppressed, as is common in learning theory.}. $\mathcal{H}$ denotes the class of local classifiers, and $\mathcal{G}$ the class of cloud classifiers. We use $g\in \mathcal{G}$ to  denote the high-complexity `cloud' classifier. The training set is taken to be $\{(X_i, g(X_i)\}$, where the $X_i$ are assumed to have been sampled independently and identically from an unknown probability measure $\mu$ on $\mathcal{X}$.\footnote{If instead we have a raw dataset and no $g$, we assume that $g$ is obtained by training a function in $\mathcal{G}$ over this set.} For feasibility of various programs (particularly Def.~\ref{def:osl}), we assume that $ \{0, 1\} \subset \mathcal{H},$ and that  $h \in \mathcal{H} \iff 1 - h \in \mathcal{H}.$ %

The main problem is to learn approximations to $g$ in $\mathcal{H},$ with the option to `fall back' to $g.$ We aim at retaining high accuracy w.r.t.~$g$ while minimising usage of $g$ itself. 

\subsection{Bracketing for Budget Learning} 

\begin{defi}
Given a measure $\mu$ and functions, $h_1 \le h_2,$ the bracket $[h_1, h_2]$ is the set of all $\{0,1\}$-valued functions $f$ such that $h_1 \le f \le h_2$ $\mu$-a.s. The $\mu$-size of such a bracket is \( | [h_1, h_2]|_\mu := \mu(h_1 \neq h_2).\)
\end{defi}
As an example, on $\mathcal[0,1],$ the functions $0(x)$ and $\ind{x > \nicefrac{1}{2}}$ induce the bracket containing all functions that are $0$ on $[0,\nicefrac{1}{2}].$ This bracket has size $\mu(X > \nicefrac{1}{2})$.%

Notice that if $h_1 \neq h_2$ in the above, it is forced that $h_1 = 0, h_2 = 1.$ We will be concerned with the brackets that can be built using $h$s from the local class $\mathcal{H}$. \begin{defi}
    The set of brackets generated by a class $\mathcal{H}$ is \( \{ [h_1, h_2] : h_1\le h_2 , h_1, h_2 \in \mathcal{H} \}.\) We also say that these are $\mathcal{H}$-brackets.\vspace{-2pt}
\end{defi}
Suppose we can find a bracket $[h^-, h^+]$ in $\mathcal{H}$ that contains $g$. Since $h^- = h^+$ forces $g$ to take the same value, we offer the classifier \[  c_{[h^-, h^+]}(x) = \begin{cases}  h^+(x) & \textrm{ if } h^+(x) = h^-(x) \\ g(x) & \textrm{ if } h^+(x) \neq h^-(x) \end{cases}.\]
The above has the \emph{usage} $|[h^-, h^+]|_\mu.$ The budget needed by a class $\mathcal{H}$ to bracket $(g, \mu)$ is the smallest such usage,\[  \mathsf{B}(g, \mu, \mathcal{H}) := \inf_{\mathcal{H}-\mathrm{brackets}} \{ | [h^-, h^+] |_\mu: g \in [h^-, h^+]\}. \] This extends naturally to bracketing of sets. \begin{def_num}\label{def:bracket_appx}
    A set of function-measure pairs $\mathcal{S} = \{ (g_i,\mu_i) \}$ is bracket-approximable by a class $\mathcal{H}$ if for every $(g,\mu),$ there exists a $\mathcal{H}$-bracket containing $g$. The budget required for bracket approximation of $\mathcal{S}$ by $\mathcal{H}$ is \[ 
    \bud(\mathcal{S}, \mathcal{H}) := \sup_{ (g,\mu) \in \mathcal{S}} \bud(g,\mu, \mathcal{H}).\] 
\end{def_num} This is a very weak notion of approximation - all it demands is that for every $g$, we can find some $\mathcal{H}$-bracket. Typical study of bracketings concentrates on real valued functions, and studies how many brackets, or how large an $\mathcal{H}$, we need to make the loss $\bud$ smaller than some given value. We defer such explorations to \S\ref{sec:bud_learn_def}, where we define a notion of budget learning. 

For the following discussion, it is useful to define a relaxed version of brackets. \begin{defi}
Let $\alpha \in [0,1],$ and $h_1, h_2$ be $\{0,1\}$-valued functions such that $\mu( h_1 \le h_2) \ge 1 - \nicefrac{\alpha}{2}$. The $\alpha$-approximate bracket $[h_1, h_2]$ with respect to $\mu$ is the set of functions $f$ such that $\mu( h_1 \le f \le h_2) \ge 1- \alpha$. We call $1-\alpha$ the accuracy of the bracketing.
\end{defi} 
The above brackets are approximate in two ways: the order of $h_1$ and $h_2$ may be reversed, and the functions in the $[h_1, h_2]$ may leak out from within them.

\subsection{One-sided Approximation and Decoupled Optimisation of Brackets}\label{sec:decoupled_osl_for_bl}

In order to discuss the decoupled optimisation of brackets, we introduce the notion of \emph{one-sided approximation}. 

\begin{def_num}\label{def:osl}
    For a function-measure pair $(g,\mu),$ an \emph{approximation from below} to $g$ in a class $\mathcal{H}$ is any minimiser of the following optimisation problem \[  \mathsf{L}(g, \mu, \mathcal{H}) := \inf \{ \mu(h \neq g) : h \in \mathcal{H}, h \le g\}.\] We refer to $\mathsf{L}$ as the inefficiency of approximation from below of $(g,\mu)$ by $\mathcal{H}$.
    We analogously define \emph{approximation from above} as $1-h,$ where $h$ is an approximation of $1-g$ from below. 
\end{def_num}
We use `one-sided approximation' to refer to both approximation from above and below. 

If we let $h^-$ be an approximation of a function $g$ from below, and $h^+$ an approximation from above, then it follows that $h^- \le g \le h^+$. Thus, the bracket $[h^-, h^+]$ is well-defined. Further, for any bracket containing $g$, 
\vspace{-5pt}
\begin{align*}
    \mu( h^+ \neq h^-) &= \mu(h^+ \neq h^-, g = 1) + \mu(h^+ \neq h^-, g = 0) \\
                       &= \mu(h^- = 0, g = 1) + \mu( h^+ =1 , g = 0) \\
                       &= \mu(h^- \neq g) + \mu(h^+ \neq g).
\end{align*} 
\noindent Thus, if $h^+$ and $h^-$ are respectively the minimisers of the right hand side, they must also be minimisers of the left hand side. Immediately, we have \[  \mathsf{B}(g, \mu, \mathcal{H}) = \mathsf{L}(g, \mu, \mathcal{H}) + \mathsf{L}(1-g, \mu, \mathcal{H}),\] and the respective minimisers of the $\mathsf{L}$s form a $\mu$-optimal $\mathcal{H}$-bracketing of $g$!

This means that in order to bracket $g$ optimally, it suffices to \emph{separately} learn approximations to $g$ from above and below. This decouples the optimisation problems inherent in learning these, and allows easy convex relaxations of both the above problems. 

Note that the reverse direction trivially holds - the optimal bracket containing $g$ provides two functions which upper and lower approximate $g$. These functions are optimal for the respective OSL problems.
\subsection{Convex Surrogates and ERM}\label{sec:convex_osl_&_bud_pipeline}
\S\ref{sec:decoupled_osl_for_bl} suggests one-sided learning as a method for learning bracketings. However, in an ML context, the optimisation problem of Def.~\ref{def:osl} is meaningless since $\mu$ is not available. We approach this via empirical risk minimisation (ERM) (see also \S\ref{sec:os_learnability},\S\ref{sec:osl-vc-prop}).\vspace{2pt}\\
To handle the intractable $0-1$ loss, we take the standard approach of relaxing the $h$ to take values in $[0,1]$, subsequently thresholded to get a binary function, and the loss to a convex surrogate $\ell$. We let $\theta$ be a parameterisation of $h$.\vspace{2pt}\\
Importantly, in a practical context, while solutions that are always below $g$ may be limited, a slight relaxation to `nearly always' below can yield tenable classifiers. Adopting this view, we also relax the constraint, possibly by a different surrogate $\ell'$\footnote{e.g.~$\ell'$ may grow faster than $\ell$ to minimise leakage}, and allow an explicit user determined leakage constraint $\zeta$.\vspace{2pt}\\
Finally, as is standard, we propose solving a Lagrangian form of the resulting optimisation problem via SGD over $\theta$. This gives the practical program \begin{equation}\label{one_sided_relaxed_with_data}
    \underset{\theta}{\min} \sum_{i : g(x_i) = 1} \frac{\ell(1-h_\theta(x_i))}{n_1} + \xi \sum_{i: g(x_i) = 0}  \frac{\ell'(h_\theta(x_i))}{n_0},
\end{equation} with a Lagrange multiplier $\xi,$ and where $n_b = |\{i :g(X_i) = b\}|$ for $b \in \{0,1\}$. %

The resulting bracketing scheme is as follows. The user may specify $(\ell, \ell')$, and a leakage constraint $\zeta \in (0,1].$ Each $\xi$ in (\ref{one_sided_relaxed_with_data}) yields a solution $\theta_\xi$. We propose scanning over $\xi \in \Xi,$ for some gridding $\Xi.$ Next, for each $\xi,$ we utilise a validation set $V$ to compute the empirical means $\hat{\mu}_V(h_{\theta_\xi} \neq g)$ and $\hat{\mu}_V( h_{\theta_\xi} = 1, g = 0),$ and select the $\theta_\xi$ which minimises the first, subject to the second being smaller than $\zeta/2 - \overline{\textrm{Bin}}(|V|, \zeta/2, \delta/2|\Xi|),$ where $\overline{\mathrm{Bin}}$ is the binomial tail inversion function as studied by \cite{langford2005tutorial}. Such a selection gives a $h^-.$ Similarly, we may learn an approximation from above $h^+$. Notice that with probability at least $(1-\delta),$ the $[h^-,h^+]$ so constructed is a $\zeta$-approximate bracket that contains $g$ (and thus has accuracy at least $1-\zeta$ w.r.t.~$g$).

\paragraph{Multi-class Extensions} In passing, we point out that our framework can be extended to multi-class setting. For an $M$-class setting, we may represent $g$ as the one-hot encoding $(g_1, \dots, g_M)$. Consistency in $g$ demands that $\sum g_i = 1$. We may learn lower-approximations $h_i$ to each $g_i$, and predict when only one of the $h_i$ is 1. One issue is that this leads to identifying class-specific leakage-levels, which then need to be optimised globally to achieve usage constraints. 

\subsection{Comparison to Gating Formulation}\label{sec:gating} %

The BL problem is typically formulated as simultaneously learning a gating function $\gamma$ and a local predictor $\pi$, so that for a point $x$, if $\gamma = 1,$ we predict locally using $\pi$, and if $\gamma = 0,$ we instead call the function $g$. This yields usage $\mu(\gamma = 0)$ for the overall classifier \[ 
c_{\gamma, \pi}(x) =  \pi(x) \gamma (x) + g(x) ( 1-\gamma(x)).\]
Notice that bracketing is in fact a restricted form of gating and prediction - the gate $\gamma = \ind{h^+ = h^-},$ and the predictor (say) $h^+$. In fact these are essentially identical in their expressive power for a given `richness': Suppose one learns gating and predictor functions $\gamma$ and $\pi$ from classes ${\Gamma}$ and ${\Pi}$ respectively\footnote{Observe that these must have comparable complexities, since they are both to be implemented on the same system.} Given this,
\[  h^+ := \gamma \cdot \pi + 1 - \gamma; \qquad h^- :=  \gamma \cdot \pi \]
bracket $g$ with the same usage\footnote{ if $c_{\gamma, \pi}$ has accuracy $a<1$, then these form an approximate bracket of the same accuracy.}. Crucially, the class of functions $\mathcal{H}$ generated by doing the above for every $(\gamma, \pi) \in \Gamma \times \Pi$ is a class of complexity equivalent to that of the pair $(\Gamma, \Pi),$ since it can be described by the same pair. Thus there is \emph{no loss of expressivity} in restricting attention to the bracketing setup.

\subsection{A Summary of the Conclusions}

The sections above establish the core of this paper via two formal reductions. The following statement encapsulates these. 

\begin{myth} The bracketing formulation of budget learning is equivalent to the gating formulation. Further, solving the bracketing problem is equivalent to solving the two decoupled one-sided learning problems of learning from below and from above. \end{myth}

This statement forms the core of this paper, and justifies all further explorations. Since the bracketing formulation is equivalent, we may define budget learnability via it. Further, finite sample analyses for the BL problem may be carried out via the one-sided learning problems. 

\section{Learnability}

As mentioned in the previous paragraph, we define notions of one-sided and budget learnability.%
\subsection{One-sided Learnability}\label{sec:os_learnability}
With only finite data, it is impossible to certify that $h \le f$ for most $h$, rendering the one-sided constraint tricky. We take the PAC approach, and relax this condition by introducing a `leakage parameter' $\lambda$.   
    \begin{def_num}
        A class $\mathcal{H}$ is \emph{one-sided learnable} if for all $(\varepsilon, \delta, \lambda) \in (0,1)^3$, there exists a $m(\varepsilon, \delta, \lambda, \mathcal{H}) < \infty,$ and a scheme $\mathscr{A}: (\mathcal{X} \times \{0,1\})^m \to \mathcal{H}$ such that for any function-measure pair $(g,\mu)$, given $m$ samples of $(X_i, g(X_i) ),$ with $X_i \overset{\textrm{i.i.d.}}{\sim} \mu$, $\mathscr{A}$ produces a function $h \in \mathcal{H}$ such that with probability at least $1-\delta$: \begin{align*} 
            &\mu( g(X) = 0, h(X) = 1) \le \lambda \\
            &\mu( g(X) = 1, h(X) = 0) \le \mathsf{L}(g, \mathcal{H}, \mu) + \varepsilon.
        \end{align*}
    \end{def_num}%
    The above definition closely follows that of PAC learning in the agnostic setting, with the deviations that leakage is explicitly controlled, and that the excess risk control, $\varepsilon$, is on $\mathsf{L}$, i.e.~it is only with respect to \emph{entirely} non-leaking functions. A key shared feature is that one-sided learnability is a property only of the class $\mathcal{H}$, and is agnostic to $(g, \mu).$ 
    
    If the class $\mathcal{H}$ is learnable, then with $m(\varepsilon, \lambda, \delta, \mathcal{H})$ samples we may learn a approximate-bracketing of any $g$ with usage at most $\bud(g, \mu, \mathcal{H}) + 2(\varepsilon + \lambda) $ and accuracy at least $ 1 - 2\lambda$
    
    Let us distinguish the above from the Neyman-Pearson classification setting of \cite{cannon2002learning, scott2005neyman}. The latter can be seen as learning from below, but with explicit control on the \emph{conditional probability} $\mu(h = 1| g= 0).$\footnote{In addition, the targeted control on this is some level $\alpha >0$, not $0$, and a relaxation of the form we use to $\alpha + \lambda$ is also utilised. Further, the property of only comparing against the best classifier at the target level of leakage ($\alpha$ in their case, $0$ in ours) is also shared.} This is too strong for our needs - we are only interested in emulating the behaviour of $g$ \emph{with respect to $\mu$}, and so if $\mu(g = 0) <\lambda,$ then it is fine for us to learn any $h$. This induces the difference that the error rates in the cited papers decay with $\min(n_0, n_1)$, while our setting is simpler and PAC guarantees follow the entire sample size. Nevertheless, our claims on the sample complexity(\S\ref{sec:osl-vc-prop}) are derived similarly to the setting of `NP-ERM' in these papers, including a testing and an optimisation phase.
    
\subsection{Budget Learnability}\label{sec:bud_learn_def}

The bracket-approximation of Def.~\ref{def:bracket_appx} suffers from two problems in the ML context. Firstly, approximation by classes that are \emph{not} one-sided learnable is irrelevant. Secondly, the definition does not control for effectiveness: a bracket-approximation with  $\bud(\mathcal{S}, \mathcal{H}) = 1,$ is not useful - indeed, the trivial class $\mathcal{H} = \{0(x), 1(x)\}$ attains this for every $\mathcal{S}.$ We propose the following to remedy these.
\begin{def_num}\label{def:bud_learn_single_func}
    We say that a set of function-measure pairs $\mathcal{S} = \{ (g,\mu), \dots \}$ is budget-learnable by a class $\mathcal{H}$ if $\mathcal{H}$ is one-sided learnable and \( \bud(\mathcal{S}, \mathcal{H}) < 1. \)\\We also, say that $\mathcal{H}$ can budget learn $\mathcal{S},$ adding ``with budget $\bud$'' if $\bud(\mathcal{S}, \mathcal{H})  \le \bud.$
\end{def_num}
Learning theoretic settings usually require measure independent guarantees, leading to
\begin{def_num}
    A function class $\mathcal{G}$ on the measurable space $(\mathcal{X}, \mathscr{F})$ is said to be budget learnable by a class $\mathcal{H}$ if the set $\mathcal{S} := \mathcal{G} \times \mathcal{M}$ is budget learnable by $\mathcal{H}$, where $\mathcal{M}$ is the set of all probability measures on $(\mathcal{X}, \mathscr{F})$.
\end{def_num}

Notice that strict inequality is required in Def.~\ref{def:bud_learn_single_func}. This is the weakest notion that is relevant in an ML context. Also note the trivial but useful regularity property that if $\mathcal{H}$ is one-sided learnable, then $\mathsf{B}(\mathcal{H} \times \mathcal{M}, \mathcal{H}) = 0$ - indeed, every $h \in \mathcal{H}$ is bracketed by $[h,h].$ 

\section{Theoretical Properties}

This section details some useful consequences of the above definitions, which serve to highlight their utility. 

\subsection{One-sided learnability}\label{sec:osl-vc-prop}

Standard PAC-learning is intrinsically linked to the VC-dimension. The same holds for one-sided learnability.

\begin{myth}\label{thm:one-sided-learn_samp}
        If $\mathcal{H}$ has finite VC-dimension $d$, then it is one-sided learnable with \[  m(\varepsilon, \lambda, \delta, \mathcal{H}) = \widetilde{O} \left( \left( \frac{1}{\lambda} + \frac{1}{\varepsilon^2} \right) (d   + \log(1/\delta) )\right).\]
        Conversely, if $\mathcal{H}$ is one-sided learnable and has VC-dimension $d > 1$, then for $\delta< 1/100,$ \[  m(\varepsilon, \lambda, \delta,\mathcal{H}) > \frac{d-1}{32(\lambda + \varepsilon)} .\] Particularly, one-sided learnable classes must have finite VC-dimension.
\end{myth}

The proof is left to Appx.~\ref{appx:osl_pac_proofs}. The lower bound is proved via a reduction to realisable PAC learning, while the upper bound's proof is similar to that for agnostic PAC learning, with the  modification of adding a test that eliminates functions that leak too much.

The point of the Theorem \ref{thm:one-sided-learn_samp} is to illustrate that sample complexity analyses for our formulation can be derived via standard approaches in learning theory. Alternate analyses via, e.g., Rademacher complexty or covering numbers are also straightforward (Appx.~\ref{appx:rad-cov}).

\subsection{Budget Learnability}\label{sec:bud_learnability_bones}

The key question of budget learning is one of bias: what classes of functions can be budget learned by low complexity classes? This section offers some partial results towards an answer.

Before we begin, the (big) question of how one measures complexity itself remains. We take a simple approach - since one-sided learnability itself requires finite VC-dimension, we call $\mathcal{H}$ low complexity if $\vc(\mathcal{H})$ is small. Certainly VC dimension is a crude notion of complexity. Nevertheless this study leads to interesting bounds, and outlines how one may give theoretical analyses for more realistic settings that may be pursued in further work. 

Importantly, we do not expect any one class to be able to meaningfully budget learn \emph{all} classes of a given complexity. This follows since the definition of budget learnability implies that if sets $\mathcal{S}_1, \mathcal{S}_2$ of function-measure pairs are budget learnable, then so is $\mathcal{S}_1 \cup \mathcal{S}_2.$ Such unions can lead to arbitrary increase in complexity, which must weaken the budget attained.\footnote{Formally, this finite union property and the lower bound Thm.~\ref{thm:bl_vc_lower_bounds} part $\mathrm{(i)}$ indicate that if $\mathcal{H}$ can budget learn \emph{all} classes of VC dimension $D$ on all measures with budget $1 - c$ for any $c > 0$ that depends only on $D$ or $\mathcal{X}$, but not on $\mathcal{H},$ then $\forall k \in \mathbb{N}, \vc\left(\mathcal{H}\right) \ge  C kD$ for a constant $C$. } Thus, at the very least, the classes $\mathcal{H}$ must depend on $\mathcal{G}$, although we would like them to not depend on the measure.

\subsubsection{Budget Learnability of Regular Classes}\label{sec:regular_classes_bud_learn}
The class of H\"{o}lder smooth functions is a classical regularity assumption in non-parametric statistics. In this section, we define a natural analogue for $\{0,1\}$-valued functions, and discuss its budget learnability by low VC dimension classes. For simplicity, we restrict the input domain to the compact set $\mathcal{X} = [0,1]^p.$ We use $\mathrm{Vol}$ to denote the Lebesgue measure on $\mathcal{X}.$
    \begin{defi}
        Let $g$ be a $\{0,1\}$-valued function. A partition $\mathscr{P}$ of $\mathcal{X}$ is said to be aligned with $g$ if each set $\Pi \in \mathscr{P}$ has connected interior, and if $g$ is a constant on each such set.
    \end{defi} 
    We define a notion of regularity for partitions below. Recall that a $p$-dimensional rectangle is a $p$-fold product of $1$-D intervals.
    \begin{defi}
        A partition $\mathscr{P}$ is said to be $V$-regular if every part $\Pi \in \mathscr{P}$ contains a rectangle $R_\Pi$ such that $\mathrm{Vol}(R_\Pi) \ge V$ and $\mathrm{Vol}(\Pi \setminus R_\Pi) < V$. 
    \end{defi}
    The above partitions are well aligned with rectangles in the ambient space. The notion of regularity for function classes we choose to study demands that each function in the class has an associated `nice' partition.    
    \begin{def_num}
        We say that a class of functions $\mathcal{G} = \{g:[0,1]^p \to \{0,1\}\}$ is $V$-regular if for each $g \in \mathcal{G},$ there exists a $V$-regular partition aligned with $g$.
    \end{def_num}    
    Essentially the above demands that the local structure induced by any $g$ can be neatly expressed. This condition is satisfied by many natural function classes on the bulk of their support - An important example is the class of $g$ of the form $\mathds{1}\{G(x) > 0\}$ for some H\"{o}lder smooth $G$ that admit a margin condition with respect to the Lebesgue measure (see, e.g.~\cite{mammen1999smooth, tsybakov2004optimal}). Indeed, if $\{G\}$ satisfies the margin condition $\mathrm{Vol}(|G|<t) \le \eta$, and is $L$-Lipschitz, then $\{ \ind{G > 0} \}$ is $V$-regular on a region of mass $\ge 1-\eta$ with $V \ge (2t/L)^p.$
    
    We offer the obvious class that can budget learn $V$-regular functions over sufficiently nice measures - rectangles. For $\kappa \in \mathbb{N},$ we define the class $\mathcal{R}_\kappa^{0,1}$ to consist of functions $h$ that may be parametrised by $k$ rectangles $\{R_i\}$ and a label $s \in \{0, 1\},$ and take the form \[  h( x ; \{R_i\}, s) =  s \ind{x \in \cup R_i} + (1-s) \ind{x \not\in \cup R_i}.\]   
    The class $\mathcal{R}_\kappa^{0,1}$ above has VC dimension at most $2p(\kappa + 1)$. The theorem below offers bounds on the budgets required to learn $V$-regular classes in $p$ dimensions:    
    \begin{myth}\label{thm:v_regular}
        Let $\kappa := \lfloor d/2p - 1\rfloor \le 1/V$. Suppose $\mu \ll \mathrm{Vol},$ and $\frac{\mathrm{d} \mu}{\mathrm{d}\,\mathrm{Vol}} \ge \rho,$ and $\mathcal{G}$ is $V$-regular. Then $\vc(\mathcal{R}_{\kappa}^{0,1}) = d,$ and it can budget learn $\mathcal{G} \times \{\mu\}$ with \[  \bud\left(\mathcal{G}\times \{\mu\}, \mathcal{R}_{\kappa}^{0,1}\right) \le 1 - \rho \left\lfloor \frac{d}{2p} - 1 \right\rfloor V/3.\]
        Conversely, for $V \le 1/2,$ there exists a $V$-regular class $\mathcal{G}'$ such that if $\vc\left(\mathcal{H}\right) \le d,$ then \[  \bud( \mathcal{G'} \times \{\mathrm{Vol}\}, \mathcal{H}) \ge 1 - \sqrt{3 V d\log(2e/V)}.\] 
    \end{myth}    
    For the Lipschitz functions with margin discussed above, $V$ scales as $\widetilde{\Theta}((C/L)^{-p}),$ where $L$ is the bound on the gradient, and $C$ is some constant. The above shows that all such classes are learnable with budget $1-\Omega(1)$ and VC-dim.~$d$ iff $d \gtrsim L^{-p + O(\log p )}$ 
\subsubsection{Budget Learnability of bounded VC classes}

Typically the function classes $\mathcal{G}$ that a cloud can implement are not nearly as rich as the set of all $V$-regular functions. This merits the investigation of classes with bounded (but large) complexity. Following the lines of study above, we investigate the budget learnability of finite VC classes, assuming $\vc\left(\mathcal{G}\right) =D$ for large $D$.

Unlike covering numbers, bracketing numbers do not, in general, admit control for VC classes (e.g.~constructions of \cite{vanHandel2013_arxiv} and \cite{Malykhin2012}). This renders the budget learnability problem for bounded VC classes difficult. This is further complicated by the fact that we are interested in whether such classes can be meaningfully bracketed by \emph{low-complexity} classes. Such questions are non-trivial to answer, and, frankly speaking, we do not solve the same. However, we offer two lower bounds, illustrating that if one wishes to non-trivially budget learn such classes with budget $1 - \Omega(1),$ and with VC dim.~$d$, then $d$ must grow as $\Omega(D).$ Further, we a present a few simple, natural cases where one \emph{can} budget learn, irrespective of measure, with budget $\approx 1 - d/D.$ We briefly discuss an open question that these classes stimulate.

\subsection{Lower Bounds}
For simplicity, we assume that $\mathcal{X} = [1:N]$ for some $N \gg 1,$ and that $\mathscr{F} = 2^{\mathcal{X}}.$ The classes $\mathcal{G}, \mathcal{H}$ can then be identified as members of $2^{\mathscr{F}}.$
Our lower bounds are captured by the following statements\newpage

\begin{myth}\label{thm:bl_vc_lower_bounds}$ $
\begin{enumerate}[topsep = 0pt, parsep = 1pt, itemsep = 0pt]
    \item[\emph{(i)}] (Varying measure) Let $\mathcal{G}$ be any class with $\vc(\mathcal{G}) = D,$ and $\mathcal{H}$ with $\vc(\mathcal{H}) = d$. Then there exists a measure $\mu$ such that \[ \bud(\mathcal{G} \times \{\mu\}, \mathcal{H})  \ge 1 - \sqrt{ 3\frac{d}{D} \log \frac{eD}{d}}.\]
    \item[\emph{(ii)}] (Uniform measure) Let $N \in \mathbb{N},$ be a multiple of $D$ such that $D \le N/8e.$ There exists a class $\mathcal{G}$ of VC-dimension $D$ on $[1:N]$such that for any class $\mathcal{H},$ if $\bud\left( \mathcal{G} \times \{ \mathrm{Unif}([1:N])\}, \mathcal{H}\right) \le \bud \in (D/N, 1/4e),$ then \[ \vc\left( \mathcal{H}\right) \ge D \frac{\log(1/4e\bud)}{\log (e N)}.\]  
\end{enumerate}

\end{myth}

The above bounds, while not very effective, indicate that to get small budget it is necessary that $d$ grows linearly with $D$. %

\subsection{Some natural budget learnable classes}\label{sec:bud_conj}

We present three simple examples: \begin{itemize}
    \item Sparse VC class: on the space $\mathcal{X} = [1:N],$ let $\mathcal{G} = \binom{\mathcal{X}}{\le D}$. Then this $\mathcal{G}$ can be budget learned by the class $\binom{\mathcal{X}}{\le d}$ of VC dimension $d$ with budget $1 - d/D.$
   \item Convex Polygons in the plane: Let $\mathcal{X} = \mathbb{R}^2,$ and $\mathcal{P}_D$ be the set of concepts defined by marking the convex hull of any $D$ points as $b \in \{0,1\},$ and its exterior by $1-b.$ \cite{Takacs07thevapnik-chervonenkis} shows that $\mathcal{P}_D$ has VC dimension $2D + 2$. For $d \ge 4,$ the class $\mathcal{P}_d$ (of VC dimension $2d + 2$) can budget learn $\mathcal{P}_D$ with budget $ 1 - \lceil \frac{D}{d-2}\rceil^{-1} \approx 1 - (d/D)$ for $D \gg d$.
    \item Tensorisation of thresholds: Let $\mathcal{X} = [1:N],$ and let $\mathcal{G}$ be defined as the following class: Let $\mathcal{G}_0$ be the class on $[1:N/D]$ of the form $\ind{x \ge k}$ for some $k$. We let $\mathcal{G} = \sum_{i = 1}^D g_i$ where $g_i: [1 + iN/D, (i+1)N/D] \to \{0,1\}$ are of the form $g_i(x) = g_i'(x - iN/D)$ for some $g_i' \in \mathcal{G}_0.$ Again, there exists a $\mathcal{H} \subset \mathcal{G}$ of VC dimension $d$ that can budget learn $\mathcal{G}$ with budget $1 - d/D.$   
\end{itemize}
Proofs for the above claims are left to Appendix \ref{appx:proofs_vc_bounds}. There are two important features of the above classes, and their budget approximation
\begin{enumerate}
    \item For each of the classes, there is a \emph{subset} of these classes that has small VC dimension and can budget learn at (roughly) the budget $1 - d/D$. This subclass can be chosen irrespective of measure.
    \item These classes are all extremal in the sense of satisfying the sandwich lemma with equality. In the first two cases they are maximal, while the third class is ample (see, e.g.~\cite{chalopin2018unlabeled}).    
\end{enumerate}
Maximal classes are known to admit unlabelled compression schemes of size equal to their VC dimension, and have many regularity properties - for instance, subclasses formed by restricting the class to some subset of the input are also maximal (see \cite{chalopin2018unlabeled} and references within). It is an interesting open question whether maximal classes of VC dimension $D$ can be budget learned by \emph{subclasses} of VC dimension $d$ with usage $1- cd/D$ for some constant $c$.
\section{Experiments}\label{sec:exp_main}

This section presents empirical work implmenting the BL via bracketing schema on standard machine learning data. We explore three binary classification tasks\begin{enumerate}
    \item A simple \emph{synthetic} task in $\mathbb{R}^2$ that allows easy visualization.
    \item The \emph{MNIST odd/even} task, which requires discrimination between odd and even MNIST digits.
    \item The \emph{CIFAR random pair} task, which requires discrimination between a pair of randomly chosen CIFAR-10 classes.\footnote{Note: supervision is provided \emph{after} this choice. That is, if class $a$ and $b$ are chosen, then the algorithms are provided the class $a$ and class $b$ data.}
\end{enumerate}

 The models considered are presented in Table \ref{table:tasks_and_classes}. Each of the local classes chosen are far sparser than the corresponding cloud classes, which are taken to be the state of the art models for these tasks. 
 \begin{table}[h]
 \centering
\resizebox{\columnwidth}{!}{%

 \begin{tabular}[b]{c||cc|cc}
\multicolumn{1}{c}{\multirow{2}{*}{\textbf{\begin{tabular}[c]{@{}c@{}}Task \end{tabular}}}} & \multicolumn{1}{c}{\multirow{2}{*}{\textbf{\begin{tabular}[c]{@{}c@{}}Cloud \\ Classifier \end{tabular}}}} & \multicolumn{1}{c}{\multirow{2}{*}{\textbf{\begin{tabular}[c]{@{}c@{}}Cloud \\ Accuracy \end{tabular}}}} & \multicolumn{1}{c}{\multirow{2}{*}{\textbf{\begin{tabular}[c]{@{}c@{}}Local \\ Classifier \end{tabular}}}} & \multicolumn{1}{c}{\multirow{2}{*}{\textbf{\begin{tabular}[c]{@{}c@{}}Local \\ Accuracy \end{tabular}}}} \\ \\\hline \hline  & & & & \\
\multirow{3}{*}{ \textbf{ \begin{tabular}[c]{@{}c@{}} Synthetic \end{tabular} }}  & \multirow{2}{*}{ \begin{tabular}[c]{@{}c@{}} 4th order curve \end{tabular} } & \multirow{2}{*}{ 1.00 } & \multirow{2}{*}{ \begin{tabular}[c]{@{}c@{}} Axis-aligned Conic Sections \\ (2nd order curves) \end{tabular}  } & \multirow{2}{*}{ 0.840 } \\
& & & & \\ & & & & \\
\multirow{3}{*}{ \textbf{ \begin{tabular}[c]{@{}c@{}} MNIST \\ Odd/Even \end{tabular} }} &  \multirow{3}{*}{ \begin{tabular}[c]{@{}c@{}} LeNet \\ 2conv + maxpool layers\\43.7K params \end{tabular} }  & \multirow{3}{*}{ 0.995 } & \multirow{3}{*}{ \begin{tabular}[c]{@{}c@{}} Linear \\ 1.57K params \end{tabular} } & \multirow{3}{*}{ 0.898 } \\
& & & & \\ & & & & \\  & & & &\\
\multirow{3}{*}{ \textbf{ \begin{tabular}[c]{@{}c@{}} CIFAR \\ Random Pair \end{tabular} }} & \multirow{3}{*}{ \begin{tabular}[c]{@{}c@{}}RESNET-32 \\0.46M params \end{tabular} } & \multirow{3}{*}{ 0.984 } & \multirow{3}{*}{ \begin{tabular}[c]{@{}c@{}}  Narrow LeNet \\ 2conv + maxpool layers\\1.63K params \end{tabular} } & \multirow{3}{*}{ 0.909 } \\
& & & & \\  & & & &\\
\end{tabular}
 } \caption{\small Classification tasks studied, and the corresponding cloud and local classifier classes selected. Cloud accuracy is reported with respect to true labels, but local accuracy is with respect to cloud labels.  \label{table:tasks_and_classes} }%
 \end{table}

 Bracketing is implemented as described in \S\ref{sec:convex_osl_&_bud_pipeline}. See Appx.~\ref{appx:exp} for detailed descriptions. We compare the bracketing method to four existing approaches. 
\begin{enumerate}
    \item \emph{Sum relaxation} (Sum Relax.) \cite{cortes2016learning}, which relaxes the gating formulation to a sum as $\pi \gamma \le (\pi + \gamma)/2$, and then further relaxes this to real valued outputs and convex surrogate losses.
    \item \emph{Alternating Minimiation} (Alt.~Min.) \cite{nan2017adaptive}, which introduces an auxiliary function $u$ to serve as proxy for $\gamma$ during training, replacing $\gamma \pi$ by $u \pi$. The algorithm then optimises a loss over $(\gamma, \pi, u)$ via alternating minimisation over $(\gamma, \pi)$ and then $u$, using a KL penalty $D(u \|\gamma)$ to promote $u \approx \gamma$.
    \item \emph{Selective Net} (Sel.~Net.) \cite{geifman2019selectivenet}, which is an architectural modification for deep networks that essentially optimises the raw gating setup without any relaxation via SGD. %
    \item \emph{Local Thresholding} (Local Thresh.). This is a na\"{i}ve baseline - one learns a local classifier, and then rejects points if the entropy of its (soft) output at the point is too high.
\end{enumerate}

In line with the focus of the paper, we only report solutions at high target accuracy ($\ge98\%$). We note that local thresholding strictly outperforms the sum relaxation and alternating minimisation methods. The results are reported in Fig.~\ref{fig:synthetic_experiments} and Table \ref{table:usages}. Observe that the bracketing methods show a consistent gain in usages at high accuracy, with reductions in usage by a factor of 1.2 to 1.5 times over the best competitors which are local thresholding for MNIST and Sel. Net. for CIFAR. In addition, the usages themselves are in the range $20$-$40\%$ in most of the cases.

It is important to contextualise these usage numbers. In our choice of cloud and edge models, we are demanding that the edge models punch far above their weight when we try to budget learn the stated cloud classifiers - indeed, the edge models do not come even close to the clouds in standard accuracy. However, in Table \ref{table:usages}, we see usages of $20$-$40\%$ at high accuracies, and relative operational lifetimes (inverse of usage, see, e.g.~\cite{zhu_iot}) of 2.5-5. For settings like IoT devices, where communication dominates energy costs, this is a significant gain in operational lifetimes of the prediction pipeline at near SOTA accuracy. 

These results demonstrate that the bracketing methodology is practically implementable and effective, with the resulting budget learners clearly outperforming existing methods on the studied tasks. 
\begin{figure}
\centering
\begin{subfigure}{.19\textwidth}
    \centering \includegraphics[width = \textwidth]{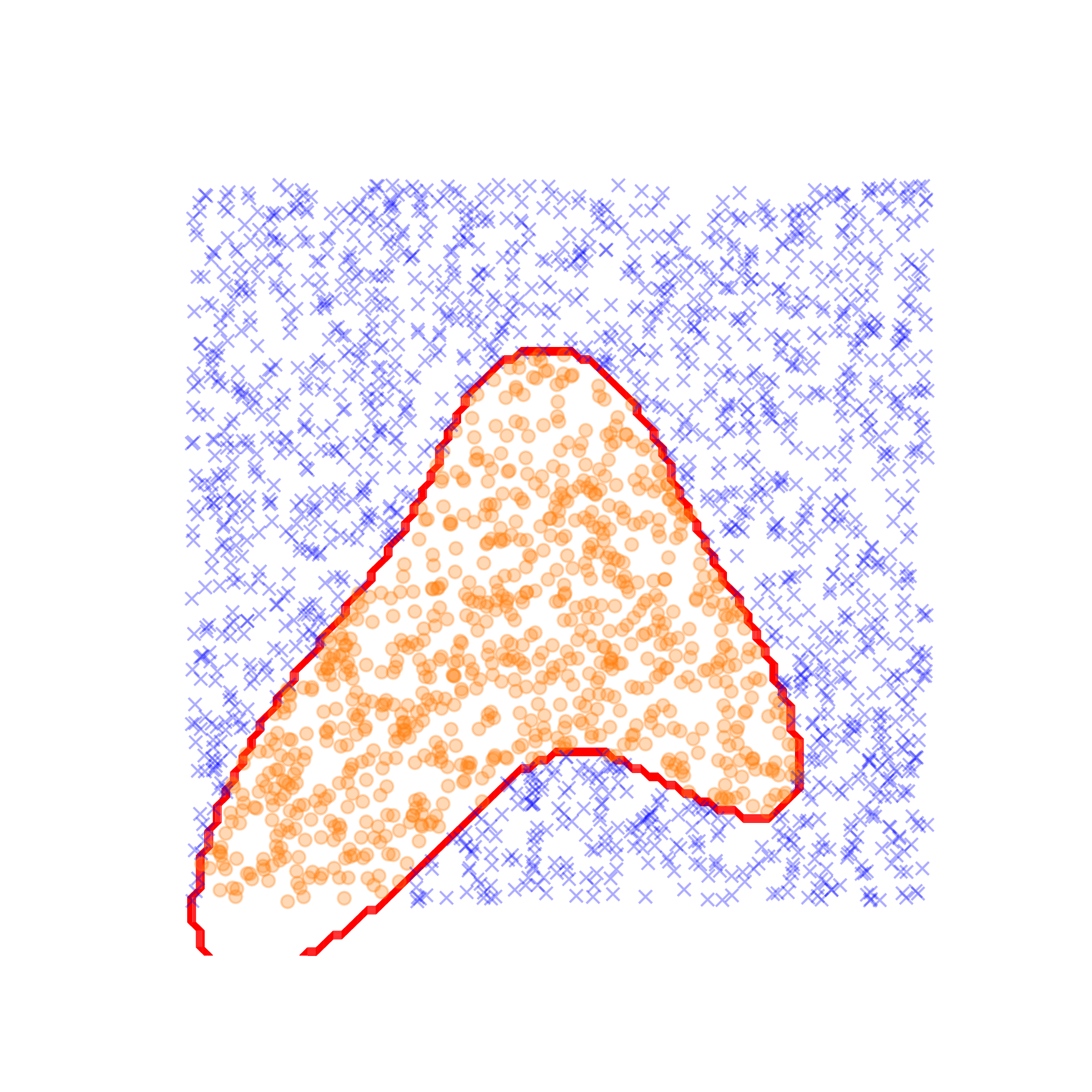}
    \caption{Cloud boundary, and the training set.}\label{fig:syn1}
\end{subfigure}   %
\begin{subfigure}{.19\textwidth}
    \centering \includegraphics[width = \textwidth]{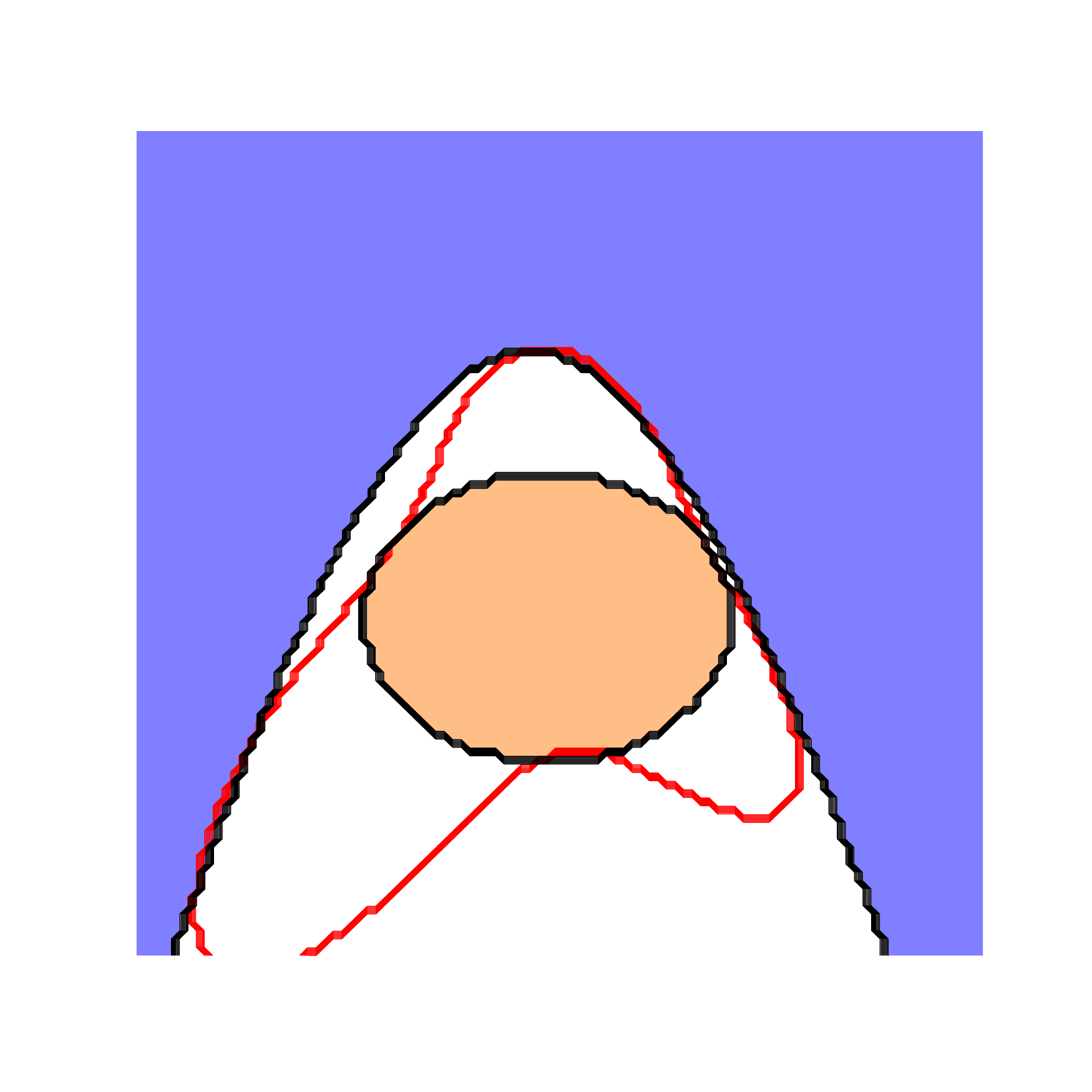}%
    \caption{Bracketing\\ (Acc: 0.997; Usg:0.295)}
\end{subfigure} 
\begin{subfigure}{.19\textwidth}
   \centering  \includegraphics[width = \textwidth]{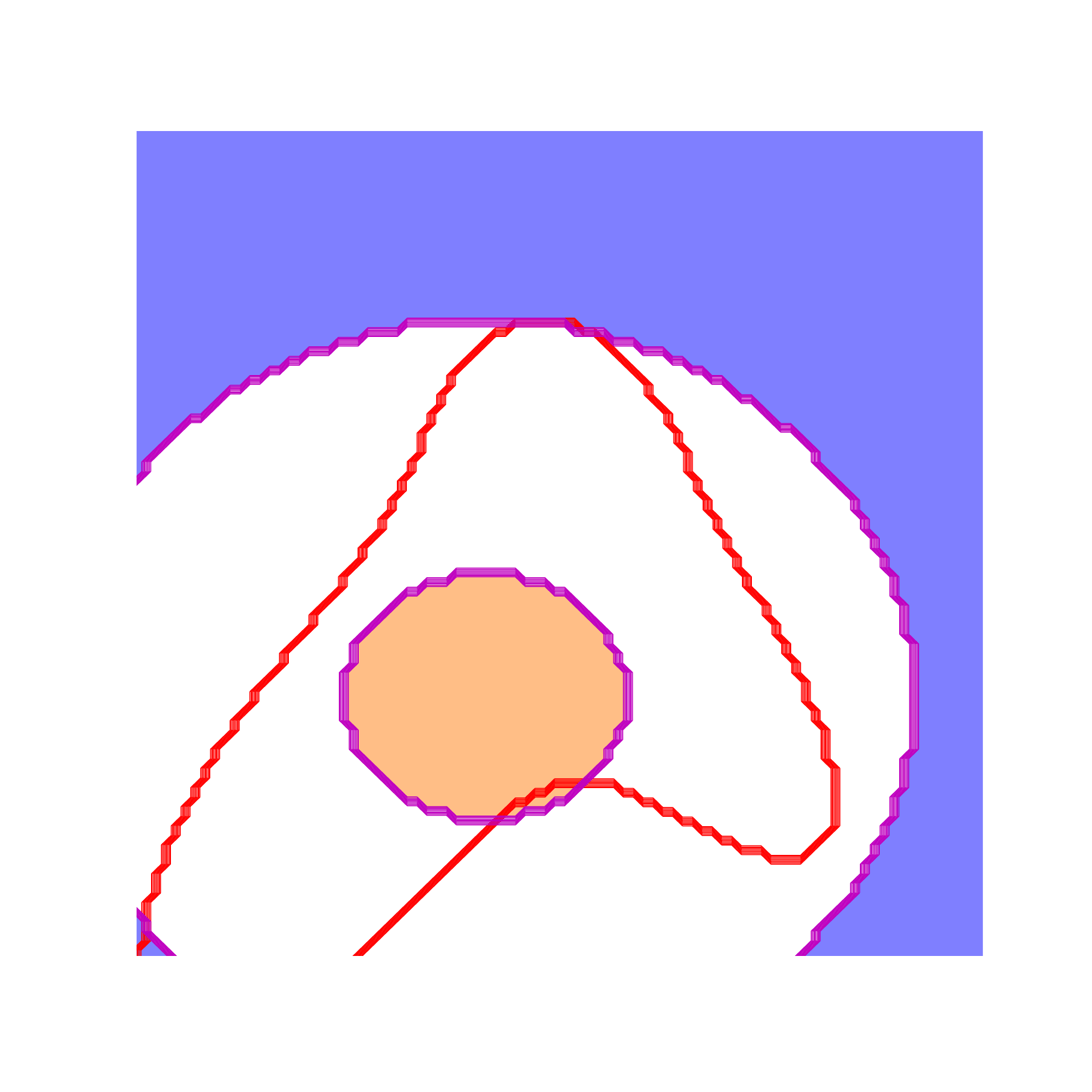}%
    \caption{Local thresholding \\ (Acc:0.997; Usg:0.537)}
\end{subfigure}   %
\begin{subfigure}{.19\textwidth}
    \centering \includegraphics[width = \textwidth]{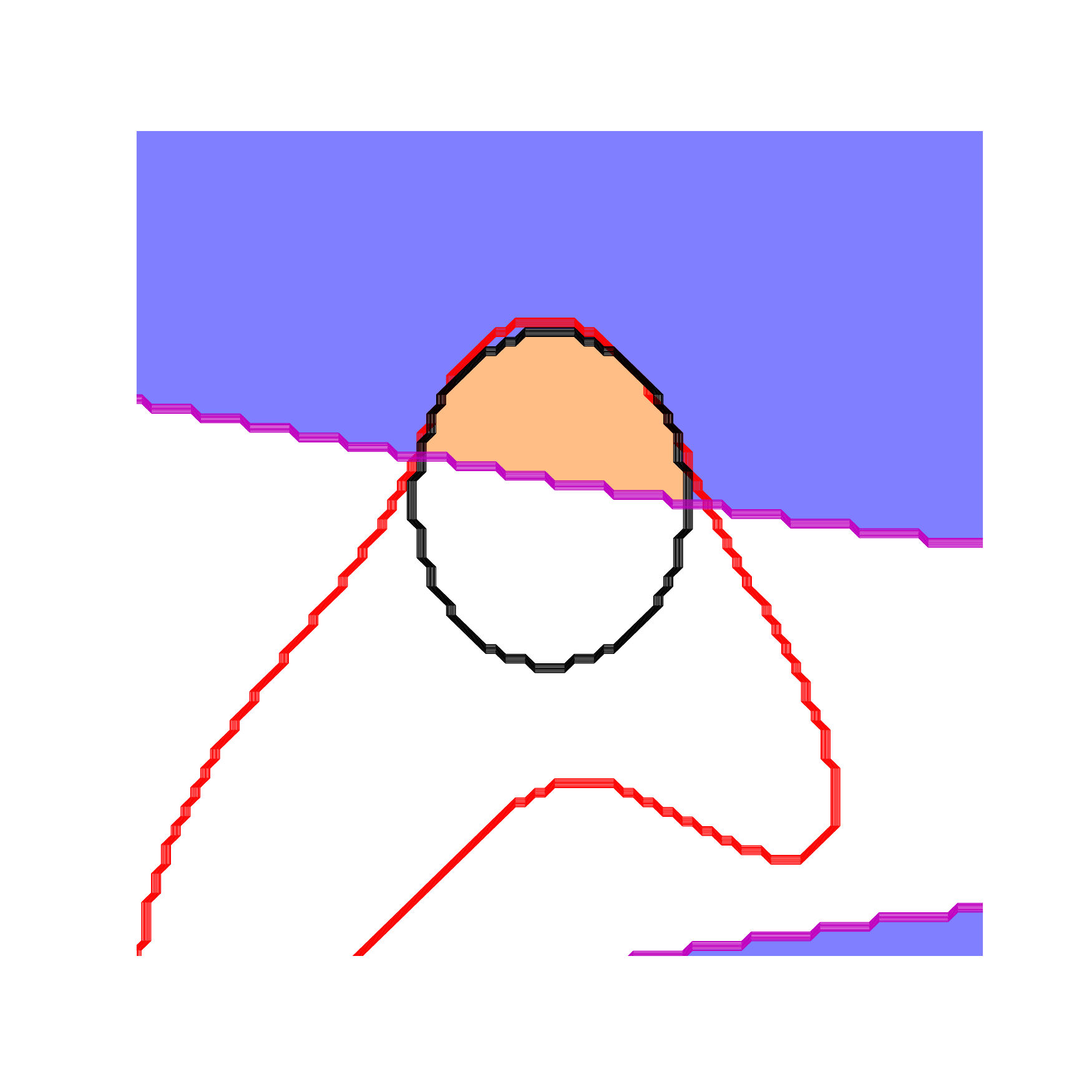}%
    \caption{Alt. Min.\\ (Acc:0.996; Usg:0.563)}
\end{subfigure}   %
\begin{subfigure}{.19\textwidth}
   \centering \includegraphics[width=\textwidth]{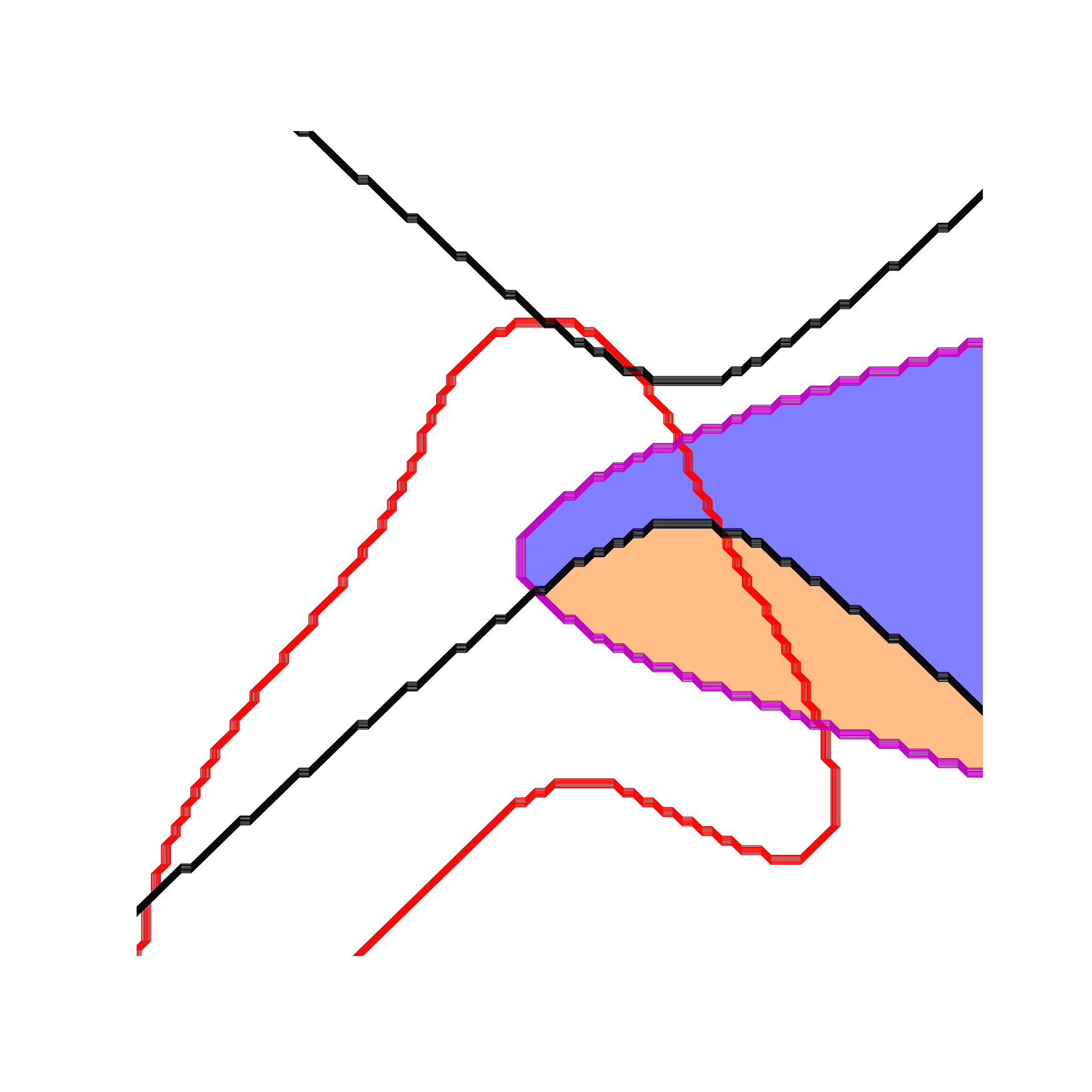}%
    \caption{Sum Relax.\\ (Acc:0.948; Usg:0.819)}
\end{subfigure}   %

\caption{\small Visualisation of classifiers resulting from the various approaches on a synthetic dataset. The red curve indicates the decision boundary of the cloud classifier, and figure (a) indicates this, and also shows the training set used as coloured dots. Figures (b)-(e) depict the budget learners learnt by various approaches. In these, the white region is the set of inputs on which the cloud is queried, while the orange and blue regions describe the decisions of the local predictor when it is queried. The black lines indicate decision boundaries of the various classifiers, and in figures (c)-(e), the magenta line indicates the boundary of the gate. Minimum usage solutions with accuracy at least $99.5\%$ (when found) are presented.}
\label{fig:synthetic_experiments}%
\end{figure}

\begin{table}
\centering
    \resizebox{\columnwidth}{!}{%

\begin{tabular}{c|c|cc|cc|cc|cc|cc|c}
\multicolumn{1}{c}{\multirow{2}{*}{\textbf{\begin{tabular}[c]{@{}c@{}} Task\end{tabular}}}} &\multicolumn{1}{c}{\multirow{2}{*}{\textbf{\begin{tabular}[c]{@{}c@{}}Target \\ Acc.\end{tabular}}}} & \multicolumn{2}{c}{\textbf{Bracketing}} & \multicolumn{2}{c}{\textbf{Local Thr.}}                              & \multicolumn{2}{c}{\textbf{Alt. Min.}} & \multicolumn{2}{c}{\textbf{Sum relax.}} & \multicolumn{2}{c}{\textbf{Sel. Net.}}& \multicolumn{1}{c}{\multirow{2}{*}{\textbf{Gain}}} \\
         \multicolumn{1}{c}{} & \multicolumn{1}{c}{\textbf{}}  & \multicolumn{1}{c}{\textbf{Usg.}} & \multicolumn{1}{c}{\textbf{ROL}}  & \multicolumn{1}{c}{\textbf{Usg.}} & \multicolumn{1}{c}{\textbf{ROL}}  & \multicolumn{1}{c}{\textbf{Usg.}} & \multicolumn{1}{c}{\textbf{ROL}}  & \multicolumn{1}{c}{\textbf{Usg.}} & \multicolumn{1}{c}{\textbf{ROL}}  & \multicolumn{1}{c}{\textbf{Usg.}} & \multicolumn{1}{c}{\textbf{ROL}}  & \multicolumn{1}{c}{}                               \\
                                    \hline\hline
\multirow{3}{*}{ \textbf{\begin{tabular}[c]{@{}c@{}} MNIST\\Odd/Even\end{tabular}} } 
& 0.995 & 0.457 & 2.19 & 0.653 & 1.53 & 0.830 & 1.20 & 0.785 & 1.27 & 0.658 & 1.52 & 1.431$\times$ \\
& 0.990 & 0.387 & 2.58 & 0.515 & 1.94 & 0.740 & 1.35 & 0.651 & 1.54 & 0.544 & 1.84 & 1.332$\times$ \\
& 0.980 & 0.299 & 3.35 & 0.358 & 2.79 & 0.604 & 1.66 & 0.651 & 1.54 & 0.423 & 2.37 & 1.199$\times$ \\ \hline
\multirow{3}{*}{ \textbf{\begin{tabular}[c]{@{}c@{}} CIFAR\\Random Pair\end{tabular}} } 
& 0.995 & 0.363 & 4.01 & 0.510 & 2.25 & 0.854 & 1.19 & 0.620 & 2.07 & 0.436 & 3.04 &  1.280$\times$\\
& 0.990 & 0.294 & 5.66 & 0.399 & 3.41 & 0.754 & 1.40 & 0.488 & 3.31 & 0.347 & 4.30 &  1.265$\times$\\
& 0.980 & 0.214 & 9.97 & 0.276 & 6.38 & 0.611 & 1.87 & 0.345 & 5.81 & 0.257 & 11.67 &  1.195$\times$\\
\end{tabular}

} \caption{\small Performances on BL tasks studied. Usage (usg.) and \emph{relative operational lifetimes} (ROL), a common metric in BL which is the inverse of usage, are reported. In each case, the models attain the target accuracy (with respect to cloud) to less than $0.5\%$ error - see Table \ref{table:usages_with_acc} in Appx.~\ref{appx:exp_tables}. \emph{Gain} is the factor by which the bracketing usages are smaller than the best competitor. The CIFAR entries are averaged over 10 runs of the random choices. The results for all runs for the best two methods for are reported in Table \ref{cifar_binary} in Appx.~\ref{appx:exp_tables}. Note that these are averages of each entry for each run and so average ROL is \emph{not} the same as the inverse of the average usage.\label{table:usages} }%
\end{table}

\section{Directions for Future Work}

We think that the bracketing formulation of BL described above is rather nice. It shows practical promise, and gives a clean framework in which to theoretically study BL. A number of problems in BL are wide open. We informally state a few of these that arise naturally from the considerations in this paper below, in the hope that you, dear reader, might want to think about them.

\emph{Extentions to Learning with Abstention}: can the bracketing approach be applied directly to LwA, in the setting where one does not want exact agreement with a given concept, but may accept a small extra risk, and perhaps for noisy data? Our suggestions for and implementation of the empirically relevant bracketing mechanisms (\S \ref{sec:convex_osl_&_bud_pipeline},\ref{sec:exp_main}) already heuristically step towards this via the explicit leakage parameter $\zeta$, but this direction must be formalised.

\emph{Extentions to Multiclass Settings}: The bracketing definitions rely intrinsically on the binary class structure. This \emph{can} be extended to multi-class settings via one-hot encoding, as suggested in \S\ref{sec:convex_osl_&_bud_pipeline}, or via constructing $\log_2 C$ binary bracketing problems using a bit encoding of the classes. Practically, however, this has the consequence of blowing up the number of Lagrange multipliers one needs to consider. To fully exploit these, an efficient way to allocate these and globally optimise them must be developed.

\emph{Practically Relevant Classes}: Modelling of the constraints at the edge and the power of the cloud can yield practically relevant settings of the classes $\cal H, \cal G$ in the above. Perhaps with these in hand, one can develop bracket approximability results that are practically relevant, and, hopefully, more optimisitc even in the worst case than the above.

\emph{Deeper Empirical Study}: of both our, and other, BL methodologies is of intrinsic and of practical interest. What are the right benchmarks and datasets for BL as studied here? How do these methods do empirically for settings that matter in practice? 

Lastly, let us mention a couple of technical problems that are insufficiently dealt with in the above. First, we remind the reader about the intriguing question about budget learnability of maximal VC classes by their subclasses. Second, our lower bounds are loose - there's a square root in them that we don't think belongs. They are also not very effective. The square-root comes from the fact that our analysis for these proceeds via covering numbers. Can bounds on bracketing numbers be given more directly, at least in simple cases? Can one remove the dependence on $N$?

\subsubsection*{Acknowledgements} Our thanks to Pengkai Zhu for help with implementing experiments. This work was supported partly by the National Science Foundation Grant  1527618, the Office of Naval Research Grant N0014-18-1-2257 and by a gift from the ARM corporation.

\printbibliography

\newpage

\begin{appendix}
\noindent \textbf{\LARGE Appendices}

\section{Proofs Omitted from the Main Text}
\subsection{Proof of Theorem \ref{thm:one-sided-learn_samp}} \label{appx:osl_pac_proofs}

\begin{proof}[Proof of lower bound] Notice that since $\mathcal{H}$ is one-sided learnable, it can learn any $h \in \mathcal{H}$ from below with $\mathsf{L} = 0.$ Thus, given $m(\varepsilon, \delta, \lambda, \mathcal{H})$, and samples $(X_i, h(X_i))$ for any $h \in \mathcal{H},$ the scheme $\mathscr{A}$ recovers a function $\widehat{h}$ such that \begin{align*} \mu( h =0, \widehat{h} = 1) &\le \lambda \\ \mu( h = 0, \widehat{h} = 1) &\le \varepsilon. \end{align*} But then $\mu(\widehat{h} \neq h) \le \lambda + \varepsilon$ - i.e. $\mathscr{A}$ also serves as a realisable PAC learner with excess risk bounded by $\lambda + \varepsilon$. Thus, standard lower bounds for realisable PAC-learning can be invoked, for instance, that of \S3.4 from the book \cite{mohribook}.
\end{proof}

\begin{proof}[Proof of Upper Bound]
    We provide a scheme showing the same. To begin with, suppose that $\mathcal{H}$ is a finite class. Fix $g, \mu$, and let $\mathcal{H}_\eta := \{ h \in \mathcal{H}: \mu( g(X)= 0, h(X) = 1) \le \eta\}$. For finite $\mathcal{H}$, the scheme proceeds in two steps:
             \begin{enumerate}
                 \item Testing: using $m_1$ samples (where $m_1$ is to be specified later), compute the empirical masses \( \widehat{\ell}(h) := \widehat{\mu}\{ h(X) = 1, g(X) = 0\} \) for every $h \in \mathcal{H}$. Let \(\widehat{\mathcal{H}}_\lambda := \{h : \ell(h) < \lambda/2\}.\)               
                 \item Optimisation: Using $m_2$ samples (where $m_2$ is to be specified later), compute the empirical masses \(\widehat{\mathsf{L}(h)} := \widehat{\mu}(h(X) = 0, g(X) = 1)\) for every $h \in \widehat{\mathcal{H}}_\lambda$. Return any \(\widehat{h} \in {\arg\!\min}_{\widehat{\mathcal{H}}_\lambda} \widehat{\mathsf{L}}(h)\). 
             \end{enumerate}
            
            The correctness of the above procedure is demonstrated by the following lemmata:
            
            \begin{mylem}\label{lemma:osl_vc_ub_testing}
                If \[ m_1 \ge \frac{24}{\lambda}  \log(4|\mathcal{H}|/\delta) ,\] then with probability at least $1-\delta/2$, \[ \mathcal{H}_{\lambda/4} \subset \widehat{\mathcal{H}}_\lambda \subset \mathcal{H}_{3\lambda/4}. \] 
            \end{mylem}
            The above is proved after the conclusion of this argument.
            \begin{mylem}\label{lemma:osl_vc_ub_opt}
                If \[ m_2 \ge \frac{2}{\varepsilon^2} \log(4|\mathcal{H}|/\delta) ,\] then with probability at least $1-\delta/2$, \[ |\widehat{\mathsf{L}(h)} - \mu(h = 0, g = 1) | \le \varepsilon\] simultaneously for all $h \in \widehat{\mathcal{H}}_\lambda.$
                \begin{proof}
                    The claim follows by Hoeffding's inequality and the union bound, noting that $|\widehat{\mathcal{H}}| \le |\mathcal{H}|.$
                \end{proof}
            \end{mylem}
            
            Thus, for finite classes, the claim follows (with $d = \log|\mathcal{H}|$) by an application of the union bound, and noting that $\mathcal{H}_0 = \{h : \mu(h = 1, g= 0) = 0\} \subset  \{ h: \mu(h = 1, g = 0) \le \lambda/4 \} = \mathcal{H}_{\lambda/4}.$
            
            We now appeal to the standard generalisation from finite classes to finite VC-dimension classes. By the Sauer-Shelah lemma (see, e.g., \S3.3 of \cite{mohribook}), with $m$ samples, a class of VC-dimension $d$ breaks into at most $(em/d)^d$ equivalence classes of functions that agree on all data points, and the losses of functions in each equivalence class can be simultaneously evaluated and share the same generalisation guarantees. Let $\mathcal{H}'$ be formed by selecting one representative from each such class. We may run the above procedure for $\mathcal{H}'$, and draw the same conclusions so long as \begin{align*}
                m &\ge m_1 + m_2 \\
                m_1 &\ge \frac{24}{\lambda} \left(d\log(em/d) + \log(4/\delta)\right) \\
                m_2 &\ge \frac{1}{2\varepsilon^2} \left(d\log(em/d) + \log(4/\delta)\right)
            \end{align*}
            
            By crudely upper bounding the right hand sides above, this can be attained if \[ \frac{m}{\log em} \ge 24 \left(\frac{1}{\lambda} + \frac{1}{\varepsilon^2} \right) \left( d  + \log(4/\delta)\right),\] and the conclusion follows on noting that for $v \ge 2, $ $u \ge 4v\log(v) \implies u/\log(eu) \ge v.$ \qedhere           
            
    \end{proof}
    It remains to show Lemma \ref{lemma:osl_vc_ub_testing}.
            \begin{proof}[Proof of Lemma \ref{lemma:osl_vc_ub_testing}]
                
                   Let $\ell(h) := \mu(h = 1, g = 0).$ Note that $m_1 \widehat{\ell(h)}$ is a $\mathrm{Binomial}(m_1, \ell(h) )$ random variable for each $h$. Further, for $p\le q,$ the distribution $\mathrm{Binomial}(n,q)$ stochastically dominates $\mathrm{Binomial}(n,p)$.  
                    
                    Thus, for any $h: \ell(h) < \lambda/4,$ \[ \mu^{\otimes m_1} (\widehat{\ell}(h) \ge \lambda/2) \le P_{U \sim \mathrm{Binomial}(m_1, \lambda/4)}(U \ge m_1\lambda/2 ) \le \exp( - 3m_1\lambda/32),\] where the final relation is due to Bernstein's inequality. 
                    
                    Similarly, for any $h : \ell(h) > 3\lambda/4,$  \[ \mu^{\otimes m_1}(\widehat{\ell(h)} \le \lambda/2) \le P_{U \sim \mathrm{Binomial}(m_1, 3\lambda/4)}(U \le m_1\lambda/2 ) \le \exp( -m_1\lambda/24).\]
                    
                    For $m_1 \ge 24/\lambda \log(4|\mathcal{H}|/\delta)$, each of the above can be further bounded by $\delta/4|\mathcal{H}|$. The claim follows by the union bound.\qedhere
                                                        
                \end{proof}

    \subsubsection{Alternate Generalisation Analyses}\label{appx:rad-cov}
    
    Note that the above proof utilises the finite VC property only to assert that on a finite sample, the hypotheses to be considered can be reduced to a finite number. Instead of the VC theoretic argument, one can then immediately give analyses via, say, $L_1$ covering numbers of the sets induced by the functions. Similarly, instead of beginning with finite hypotheses, we may instead directly uniformly control the generalisation error of the estimates for each function via the Rademacher complexity of the class $\mathcal{H},$ thus replacing Lemmas \ref{lemma:osl_vc_ub_opt}, \ref{lemma:osl_vc_ub_testing} by a bound of the form $m(\varepsilon, \lambda, \delta, \mathcal{H}) \le \inf\{m: \mathfrak{R}_m(\mathcal{H}) + \sqrt{2\log(2/\delta)/m} \le \min(\lambda, \varepsilon)/2\},$ and further extensions via empirical Rademacher complexity. In addition, one can utilise more sophisticated analyses for more sophisticated algorithms. 
    
    The point of all this is to underscore that once one adopts the bracketing and OSL setup, generalisation guarantees, and thus sample complexity bounds, follow the standard approaches in learning theory. This is not to say that these analyses may be trivial - for instance, in the above we have not shown tight sample complexity bounds at all.

    \subsection{Proof of Theorem \ref{thm:v_regular}}
    
    \begin{proof}[Proof of Upper Bound] We note that if $\frac{\mathrm{d}\mu}{\mathrm{d}\mathrm{Vol}} \ge \rho,$ and we can locally predict in a region of volume $P,$ then we can immediately locally predict in a region of $\mu$-mass $\rho P.$ Thus, it suffices to argue the claim for the Lebesgue mass on $[0,1]^p$. 
    
    Since we have access to $\kappa$ cuboids in $\mathcal{R}_{\kappa}^{0,1},$ we can capture any $\kappa$ of the cuboids induced in the minimal partition aligned with $g$ for any $g \in \mathcal{G}$. In particular, when approximating from below, we will choose $h^-$ to be $1$ on some $\kappa$ of the cuboids contained in $\{ g = 1\}$, and $0$ otherwise, and similarly for approximating from above (denoted $h^+$). Naturally, we will `capture' the cuboids with the biggest volume (more generally, biggest $\mu$ mass). Notice that this construction trivially yields $h^- \le g \le h^+$.
               
    To finish the argument, fix an arbitrary $g \in \mathcal{G}$. Let $\mathscr{P}$ be a partition aligned with $g$ that is $V$-regular, and further, has the largest total number of parts possible.\footnote{such a partition exists because $V$-regularity implies that the number of parts is uniformly bounded by $1/V$.} Suppose that there are $\mathscr{P}_1$ parts in $\mathscr{P}$ on which $g $ is 1, and $\mathscr{P}_0$ on which it is $0$. By the maximality, it must be the case that each rectangle contained in each part of $\mathscr{P}$ has volume less than $2V$, since otherwise we can split this part while maintaining $V$-regularity. Further, since the mass contained outside of the rectangle in each part is at most $V$, it follows that $3V(\mathscr{P}_0 + \mathscr{P}_1) \ge 1$ by the union bound. Thus, $\mathscr{P}_0 + \mathscr{P}_1 \ge 1/3V \ge \kappa/3.$
        
    Now, by the above construction, we can capture at least $(\min(\kappa, \mathscr{P}_0) + \min(\kappa, \mathscr{P}_1) )V$ volume of the space, which exceeds $\kappa V/3$. \qedhere
        
    \end{proof}

    \begin{proof}[Proof of Lower Bound] 
    
            Divide $[0,1]^p$ into $N = \lfloor 1/V \rfloor$ congruent, disjoint rectangles. Note that since the faces of these rectangles have codimension $\ge 1,$ they have volume $0$. Thus, we need not worry about how they are assigned in the following, and we will omit these irrelevant details in the interest of clarity. 
            
            We set $\mathcal{G}$ to be the class of $2^N$ functions obtained by colouring each of the $N$ boxes as $0$ or $1$. This class is trivially $V$-regular.
            
            Now, notice that any time a function $h$ is approximating a function $g \in \mathcal{G}$ from above, it should either attain the value $0$ on a whole box, or attain the value $1$ on a whole box - if $g$ is $1$ on a box, then $h$ is forced to be $1$. If $g$ is instead $0$, and $h$ dips down to take the value $0$ at any point, then rising up to $1$ is lossy in that it increases the loss $\mathsf{L}(h,g,\mathrm{Vol})$ while offering no reduction in the expressivity of the class $\mathcal{H}$. Thus, we may restrict attention to classes $\mathcal{H}$ such that all functions contained in them are constant over the boxes described. 
            
            Given the above setup, the entire problem is equivalently described by restricting the domains of $\mathcal{G}, \mathcal{H}$ to the centres of the above boxes, and the measure $\mathrm{Vol}$ to the uniform measure over these centres. We henceforth work in this space. The domain of the functions in $\mathcal{G},\mathcal{H}$ is now the abstract set $[1:N].$
            
            Suppose every $g \in \mathcal{G}$ can be budget learned with budget at most $1 - \Delta/N$ in this measure (where $\Delta$ is some integer because the space is discrete and the distribution is rational). Let $(h_g^+, h_g^-)$ be the appropriate bracketing functions that minimise budget for $g$, and let $\mathcal{I}_g$ be the points where $h_g^+ = h_g^-.$ The budget constraint forces that $|\mathcal{I}_g| \ge \Delta.$ Notice that outside of $\mathcal{I}_g,$ $h_g^+$ must take the value $1$ and $h_g^-$ must take the value $0$ - indeed, if $h_g^+(i) $ was $0$, then since $0\le h_g^-(i) \le h_g^+(i),$ $h_g^-(i) = 0,$ and then $i \in \mathcal{I}_g.$ 
            
            But, on $[1:N] \sim \mathcal{I}_g,$ $g$ must either be predominantly $1$ or $0$, and then respectively, must agree with $h_g^+$ or $h_g^-$ on at least $(N - |\mathcal{I}_g|)/2$ points. This means that there exists a $h'_g \in \mathcal{H}$ (which is either $h_g^+$ or $h_g^-$) such that \[ | \{i: h'_g (i) = g(i)\} \ge |\mathcal{I}_g| + \frac{N - |\mathcal{I}_g|}{2} \ge \frac{N + \Delta}{2}. \]
            
            With this setup, we invoke the following statement
            \begin{mylem}\label{lem:lowb_core}
                If a class of functions $\mathcal{F}$ on $[1:N]$ is such for every $\{0,1\}$-valued function on $[1:N],$ there exists a $f \in \mathcal{F}$ that agrees with it on at least $(N+\Delta)/2$ points, then  \[ \vc(\mathcal{F}) \ge \frac{3\Delta^2}{2(N+\Delta) \log(eN)} \ge \frac{3\Delta^2}{4N \log(eN)}.\]
            \end{mylem}
            
            Notice that since $N \ge \Delta,$ 
            Invoking the above, and the fact that the VC-dimension of ${\mathcal{H}}$ is at most $d,$ it follows that (for $N \ge 3$) \[ \frac{3\Delta^2}{8N\log N} \le d \iff \Delta \le \sqrt{3d N \log N},\] from which the claim is immediate on recalling that $1/V \ge N \ge 1/V-1$.
    \end{proof}
    
    \begin{proof}[Proof of Lemma \ref{lem:lowb_core}]
                Identify all $\{0,1\}$ labellings as above with the cube $\{0,1\}^N,$ and similarly the patterns achieved by $\mathcal{F}$ as a subset of the same. The hypothesis is then equivalent to saying that for every point $p \in \{0,1\}^N,$ there exists a point $f \in \mathcal{F}$ such that $d_{\mathrm{H}}(p,f) \le \frac{N-\Delta}{2},$ were $d_{\mathrm{H}}$ is the Hamming distance. But then $\mathcal{F}$ is a $(N-\Delta)/2$-cover of the Boolean hypercube. 
                    
                By a standard volume argument, it then must hold that \[ |\mathcal{F}| \ge \frac{2^N}{\sum_{i = 0}^{(N-\Delta)/2} \binom{N}{i}} \ge e^{+\frac{3}{2} \frac{\Delta^2}{N + \Delta}} \] where the final inequality follows on noting that the right hand side of the first inequality is 1 divided by a lower tail probability for $N$ independent fair coin flips, and then invoking Bernstein's inequality.
                    
                However, by the Sauer-Shelah Lemma, if $d \le N$ is the VC-dimension of $\mathcal{F},$ then the number of elements in it is at most \[ \sum_{i = 0}^d \binom{N}{i} \le \left(\frac{eN}{d}\right)^d.\]
                    
                Relating these, we have \begin{align*}
                       e^{+\frac{3}{2} \frac{\Delta^2}{N + \Delta}} \le \left({eN}/d\right)^d \iff  \frac{3\Delta^2}{2(N+\Delta) \log(eN/d)}  &\le d.                 \qedhere         
                \end{align*}
                                        
    \end{proof}

\subsection{Proof of Theorem \ref{thm:bl_vc_lower_bounds}}

These lower bounds are proved similarly to the lower bound from the previous section: principally, they use the fact that any non-trivial budget learner also yields non-trivial coverings, and construct function classes of limited VC dimension with large covering numbers. 

\begin{proof}[Proof of the bound $\mathrm{(i)}$]
    Let $S := \{x_1, \dots, x_D\}$ be a set of shattered points. The measure $\mu_S$ is set to the uniform distribution on $S$. The restriction $\mathcal{G}_{|S}$ consists of all $\{0,1\}$-valued functions on $D$ points. If $\mathcal{H}$ can budget learn this with respect to $\mu_S$  with budget $1 - \Delta/D,$ then $\mathcal{H}_{|S}$ is a $(D - \Delta)/2$covering of $\{0,1\}^S$. Invoking Lemma \ref{lem:lowb_core} just as in the proof of the lower bound in the previous section, we get that \( \frac{\Delta}{D} \ge \sqrt{3 \frac{\vc(\mathcal{H})}{D} \log(\frac{eD}{\vc(\mathcal{H})}) }.  \) 
\end{proof}

\begin{proof}[Proof of the bound $\mathrm{(ii)}$] 

We use a class on $[1:N],$ constructed by \cite{haussler1995sphere} that is known to have large packing number. Note that the same class is used as an example of a simple budget-learnable class in \S\ref{sec:bud_conj}. The class is defined as follows: Suppose $D$ divides $N$. Let $\mathcal{F}$ be the class of single thresholds on $[1:N/D]$, i.e.\( \mathcal{F} = \{ f_k, k \in [0:N/D+1] \},\) where $f_k(i) := \ind{ k \le i}.$ $\mathcal{F}$ trivially has a VC-dimension of $1$. $\mathcal{G}$ is generated as a tensor product of $D$ copies of $\mathcal{F}$ placed on a partition of $[1:N]$. Concretely, we may say that each $g \in \mathcal{G}$ can be represented as $D$ functions $(f_{k_1}, f_{k_2}, \dots, f_{k_D}) \in \mathcal{F}^{\otimes D}$ for some $k_1, \dots k_d \in [0:N/D + 1]$ such that for $i \in [ jN/D +1 : (j+1)N/D]$ for any $j \in [0:D-1],$ $g(i) = f_{k_j}(i).$

\cite{haussler1995sphere} shows that for this class, under the uniform measure on $[1:N],$ the $k$-packing number is at least \( \frac{ (1 + N/D)^D}{ 2^D \binom{k + D}{D}}.\) Now recall that the $k/2$-covering number must exceed the $k$-packing number for any set and metric. Further, a budget of $1 - \Delta/N$ implies a $\frac{N-\Delta}{2}$-covering. The budget requirement imposes the condition $N-\Delta \le \bud N$. Thus, invoking Sauer-Shelah as in the proof of Lemma \ref{lem:lowb_core}, we obtain \begin{align*} \left( \frac{eN}{d}\right)^d &\ge \left( \frac{N+D}{2e( N + D - \Delta)} \right)^D \\
                                                   &\ge \left( \frac{N+D}{2e( \bud N + D) } \right)^D \\
                                                   &\ge \left( \frac{1}{4e \bud}\right)^D\end{align*} 
    where we have used that $\bud N \ge D$ in the final line. The above bound is non-vacuous only if $4e\bud < 1.$
    
    The case $\bud < D/N$ is not discussed in the theorem, since it is a vanishingly small budget, but by the above, in this case we get a lower bound of $(N/4eD)^D$ in the above, giving, for $D \lesssim N^{1- \epsilon}$ for some $\epsilon > 0,$ a bound of $d = \Omega(D)$ in this setting. \qedhere
\end{proof}

\subsection{Proofs of budget claims made in \S\ref{sec:bud_conj}}\label{appx:proofs_vc_bounds}

\begin{proof}[Proof for sparse VC classes] fix any $g$. We pick the function that is $1$ on the $d$ choices of $i \in g^{-1}(1)$ with the largest total $\mu$-mass as the lower approximation, and the constant $1$ as the approximation from above.
\end{proof}

\begin{proof}[Proof for Tensorised class] The class naturally breaks the domain into $D$ equal parts, and places a threshold on each. We choose the $d-1$ parts with largest $\mu$-mass, and place a threshold there. Lastly, we collate the remaining parts into one set, and we place the constant functions $1$ and $0$ on this. A tensorisation of these function classes demonstrates the claim.
\end{proof}

\begin{proof}[Proof for Convex Polygons] 

Instead of approximation from above and from below, we will adopt the more natural terminology of inner and outer approximation. As the class is closed under $f \mapsto 1 - f,$ to show budget learnability with budget $\bud$, it suffices to show that for any polygon $P$ with $D$ vertices and any measure $\mu$, there exist polygons $P_{\mathrm{in}} \subset P \subset P_{\mathrm{out}}$ of $d$ vertices such that $\mu(P_{\mathrm{in}}) \ge  (1- \bud) \mu(P)$ and $ \mu(P_{\mathrm{out}}^c) \ge (1- \bud) \mu(P^c)$. This follows since the cloud query points are precisely those in $P_{\mathrm{out}} / P_{\mathrm{in}}),$ which has mass $ \mu(P_{\mathrm{out}}) - \mu(P_{\mathrm{in}}) \le 1 - (1-\bud)(1 - \mu(P)) - (1-\bud) \mu(P) = 1 - (1-\bud) = \bud.$ \\

\noindent \emph{Inner Approximation:} We offer a direct proof. Consecutively number the vertices of $P$ as $[1:D]$. Form the $d$-gon $P^1$ using the vertices $[1:d].$ Remove this polygon from $P$ and relabel $1 \mapsto 1, d \mapsto 2, \dots, n \mapsto n + 2 - d ,\dots.$ Contuining this process $m := \lceil D/d-2 \rceil$ times partitions $P$ into $m$ $d$-gons $P^1, \dots, P^m$. By the union bound, $\sum \mu(P^i) \ge \mu(P).$ But then there must exist at least one $d$-gon $P_{\mathrm{in}} \subset P$ such that $\mu(P_{\mathrm{in}}) \ge \mu(P) \ge \frac{1}{\lceil D/ d-2 \rceil} \mu(P).$ 

\noindent \emph{Outer Approximation:} Recall that $d\ge 4,$ and $D\ge d.$ We will show that for any $D$-gon $P$ there exists a $d$-gon $P_{\mathrm{out}}$ containing it such that $ \mu(P_{\mathrm{out}}^c) \ge \frac{d-2}{D-2} \mu(P^c).$

We induct on $D$. As a base case, for $D = d,$ the claim holds trivially since $P$ itself may serve. Let us assume the claim for $D$-gons, and let $P$ be a $D+1$-gon. Note that since $D \ge 4, D+1 \ge 5.$ Thus, $P$ has at most two pairs of consecutive exterior angles that are each exactly $\pi/2$ (since the sum of all exterior angles is $2\pi,$ and $P$ has at least $5$ exterior angles). For any side such that the two exterior angles are not both $\pi/2,$ the sides preceding and following it (in the cyclic order) may be extended to meet at some point. This yields a triangle with this side as a base. Since such an extension can be done for at least $D+1 -2 = D-1$ sides, this yields $D+1 \ge J \ge D-1$ triangles $\triangle_1, \triangle_2, \dots, \triangle_{J}.$ Now notice that for each $j \le J,$ $Q_j:= P \cup \triangle_j \supset P$ is a $D$-gon. Further, by the union bound, $\sum \mu(\triangle_j) \le \mu(P^c),$ and thus there exists a triangle $\triangle_{i^*}$ such that $\mu(\triangle_{i^*}) \le \mu(P^c)/ J,$ and thus $\mu(Q_{i^*}^c) \ge \frac{J-1}{J}\mu(P^c) \ge \frac{D-2}{D-1} \mu(P^c).$ Now, by the induction hypothesis, there exists a $d$-gon $P_{\mathrm{out}}$ containing $Q_{i^*}$ (and hence $P$) such that $\mu(P_{\mathrm{out}}^c) \ge \frac{d-2}{D-2} \mu(Q_{i^*}^c) \ge \frac{d-2}{D-1} \mu(P^c).$ This concludes the argument. 

Thus, we can attain the budget \[ \bud = 1 - \min\left( \frac{1}{\lceil \frac{D}{d - 2}\rceil}, \frac{d-2}{D-2} \right) = 1 - \left\lceil \frac{D}{d-2} \right\rceil^{-1}. \qedhere\]

\end{proof}

\section{Experiments \label{appx:exp}}

\subsection{Losses and algorithms for methods listed in \S\ref{sec:exp_main}}

We list the general approach taken for each of the methods we compare to. More precise details very between datasets, and are described in subsequent sections. Note that all models are trained on GPUs using stochastic gradient descent for linear models and ADAM for deep networks. In each case, a multitude of models are trained by scanning over values for the relevant Lagrange multiplier/regularisation weight. The collection of models so obtained is tuned, and then a model finally selected for each target accuracy via procedures detailed in \S\ref{model_selection}.

\paragraph{Bracketing} The general approach, and a formulation for generic loss functions is given in (\ref{one_sided_relaxed_with_data}) in \S\ref{sec:convex_osl_&_bud_pipeline}. The exact loss formulation used in the experiments is the following,
\begin{equation}\label{trn_loss}
     \hat{L}(\theta) = \frac{1}{N}\sum_{i=1}^{N} -1_{g(x_i)=1}\log\left(h_\theta(x_i)\right) - \xi 1_{g(x_i)=0} \log\left(1 - h_\theta(x_i)\right)
\end{equation} where $\xi$ is a hyper parameter between two components of loss function. The term multiplying $\xi$ is the constraint, which imposes a high cost in case of a leakage. The other term in the loss objective pushes the model to increase true positives. For example, if $\xi$ is $0$, local model always predicts $1$ and it has maximum leakage and minimum budget. If $\xi$ is $+\infty$, the local model always predicts $0$ and it has minimum leakage and maximum budget.

\paragraph{Local Thresholding} We first train a local predictor using the cross entropy loss and freeze it. We rank the examples based on maximum of the prediction probabilities. We select a threshold and the predictor uses cloud model if its current maximum probability is lower than threshold. We attain different budget values by changing this threshold.

\paragraph{Alternating Minimisation \cite{nan2017adaptive}} we follow the ADAPT-LIN procedure from this paper, which is an alternative minimisation scheme between an auxiliary $q$ and local predictors \& gating. Since we don't have feature costs in our setting, we assumed $\gamma=0$  in our experiments. We stopped the procedure if the $q$ vector converges, or if a predefined number of iterations - in our case 10 - is exceeded. Different budget values are obtained by sweeping values of the regularisation parameter - in this paper called $\lambda$.

\paragraph{Sum relaxation \cite{cortes2016learning}} utilising the relaxation as developed in this paper, we use the loss $L_{MH}(h,r,x,y)$ formulated within as a loss function to train a neural network. This is optimised with several values of the regularisation parameter, $c,$ to obtain different usage values.

\paragraph{Selective Net \cite{geifman2019selectivenet}} we follow the architectural augmentations and losses as prescribed by this paper. We train the network with auxiliary head and ignore this part during inference time. Again, this is performed for several values of the Lagrange multiplier, called $c$ here as well.

\subsection{Synthetic Data}

\paragraph{Cloud Classifier} A training dataset of 2.5K points was sampled uniformly from the set $[-10,10] \times [-10,10]$. The complex classifier's decision boundary can be expressed as
\[\ind{x+4x^2+3x^3+3x^4+y+y^2+y^3+y^4+5xy^2+30x^2y<1000} \]
where $x,y$ are the coordinates of the data point. 

\paragraph{Local Classifier} Weak learners are restricted to axis-aligned conic sections, which may be implemented as linear classifiers which see input features $x,y,x^2,y^2$. 

\paragraph{Training Details} Each weak learner model has hyper parameters which are adjusted to observe the power of the methods. As an example, learning rates are chosen in the range of $[10^{-5},10^{-2}]$, $\xi$ value for bracketing model is chosen in the range of $[1,3]$, $\lambda$ values for alternating minimisation are chosen in the range of $[0.25,0.75]$ and $c$ values for the sum relaxation method are chosen in the range of $[0,.3]$. After obtaining several models, the best models are reported based on the true error rates and true usages.

\subsection{MNIST Odd/Even}

\paragraph{Cloud Classifier} We implement a LeNet architecture with 6 filters in the first convolution layer, 16 filters in the second convolution layer, 120 neurons in the first fully connected layer and 84 neurons in the first fully connected layer. Kernel size for convolution layers is chosen to be 5. Overall, this model has $43.7K$ parameters. Learning rate is chosen to be $10^{-3}$ and it is halved in every 20 epochs for a total of $60$ epochs using $64$ as batch size. $L_2$ regularisation of $10^{-5}$ is applied. The model attains $99.46\%$ test accuracy. %

\paragraph{Local Classifier} Linear classifiers are adopted as weak learner architecture - these have $1.57K$ parameters, and no convolutional structure. Half of the training set ($30K$) is randomly chosen to be weak learner dataset. Within this dataset, $90\%$ ($27K$) is kept as training set for and $10\%$ ($3K$) as validation. Training and validation sets for each of the methods are kept the same to ensure a fair comparison. The local model attains $89.79\%$ test accuracy.

\paragraph{Training Details} For each model, learning rate is chosen to be $10^{-2}$ and it is halved in every $25$ epochs for a total of $120$ epochs. Batch size is chosen to be 64 and $L_2$ regularisation of $10^{-5}$ is applied. For bracketing, $\xi$ values are chosen in the range of $[0,24]$ for a total of $21$ values. For alternating minimisation,  $\lambda$ values are swept in the range $[0,1]$ for a total of $25$ values and a maximum of 10 alternative minimisation rounds are allowed. For the sum relaxation, $c$ is chosen in the range $[0,.495]$ for a total of $25$ values. For the selective net, $c$ values are chosen in range $[0,1]$ for a total of $25$ values. We note here that the auxiliary head in the selective net, which serves in deep networks as a way to improve feature extraction, is ineffective in this linear setting.

\subsection{CIFAR Random Pair}

\paragraph{Cloud Classifier} We pick ResNet32 \cite{resnet} as the high-powered model and trained it, with configurations as described by \cite{resnet_implementation}, on the full multi-class CIFAR training data. This model has $.46M$ parameters. 

\paragraph{Local Classifiers} We pick a narrow LeNet model as weak learner that has 3 filters in the first and second convolution layers, and 15 neurons in the first fully connected layer. Kernel size for convolution layers is chosen to be 5. Overall, this weak model has $1,628$ parameters. 

\paragraph{Procedure for training} For each run, we choose 2 classes out of 10 CIFAR classes randomly and extract the subset of the dataset corresponding to this couple. The cloud classifier is obtained using the pre-trained ResNet32 and only retraining the prediction layer while keeping the backbone frozen for this binary dataset. Learning rate is chosen to be $10^{-2}$ and it is halved after $50$ epochs for a total of $100$ epochs. Batch size is chosen to be $64$ and $L_2$ regularisation of $10^{-5}$ is applied. The model attains on average $98.38\%$ test accuracy. 
 
For the weak learners, $60\%$ (6K points) of the training set is randomly chosen to be the training dataset. From this, $83.3\%$ (5K) is kept as training set for and $16.7\%$ (1K) culled for validation. The model attains on average $90.94\%$ test accuracy. Training and validation set are kept the same across methods to have a fair comparison.

\paragraph{Training Details} Learning rate is chosen to be $10^{-3}$ and it is halved in every $75$ epochs for a total of $300$ epochs. Batch size is chosen to be 64 and $L_2$ regularisation of $10^{-5}$ is applied. For bracketing model, values for $\xi$ are chosen in the range of $[0,65]$ for a total of $36$ values. For alternating minimisation $\lambda$s are swept in the range $[0,1]$ for a total of $40$ values and a maximum of 10 alternative minimisation rounds are allowed. For the sum relaxation method $c$ values are chosen in range $[0,.495]$. For each of the above methods, all the networks are warm started using the parameters of the local model. Note each of the previous methods implement two Narrow LeNets - for bracketing these are the two one-sided learners, while for the other two, these are gates and predictors. For the selective net, $c$ values are chosen in range $[0,1]$ for a total of $40$ values. Warm starting this network leads to lowered performance than random initialisation, and so the latter values are reported. 

The above procedure is performed for 10 trials of random classes of CIFAR. These classes are listed in Table \ref{cifar_binary} below, along with usages attained for the bracketing and selective net methods in these cases. Only these two methods are reported here since they are the most competitive of the five.

\subsection{Model Selection Process} \label{model_selection}

For each value of the Lagrange multiplier/regularisation constant chosen in the above training methods, we receive a model (or a pair of models, as appropriate). Let this collection of models be $\mathcal{M}$.  These models have real valued outputs in the range $[0,1]$, and a decision needs to be extracted from these. In order to provide sufficient granularity to the models that they be able to match any required target accuracy, we vary the threshold of output value at which the models' decisions go from $0$ to $1$. This process differs in details for different methods. The tuning is performed 

\paragraph{Local Thresholding} In this case $\mathcal{M}$ is a singleton. We compute the cross entropy of the classifier's output and abstain if this cross entropy is larger than a threshold $\tau$ that is selected as follows: the values of $\tau$ considered are obtained by computing the cross entropies of the model outputs on each of the training points. On validation data, usages and accuracy are computed for the models which thresholds at each of the considered thresholds. At a given target accuracy, the value of $\tau$ which yields at least this accuracy on the validation data with the smallest usage is selected. 

\paragraph{Bracketing} Note that each $m \in \mathcal{M}_{\textrm{bracketing}}$ contains two models $(m_\textrm{below}, m_\textrm{above})$ which are respectively approximations from above and below - these may be trained with different $\xi$, thus giving a total of $|\Xi|^2$ models. Suppose the target accuracy is $1-\alpha$. Let the training data have size $T$. Using the training data, for every $i \in [0: \alpha T]$, we determine pairs of thresholds $\tau_m(i) = (\tau^m_{\mathrm{below}}(i), \tau^m_{\mathrm{above}}(i))$ such that the leakages of $(m_{\mathrm{below}}, m_\mathrm{above})$ on the training data are exactly $i/T$ each. This then gives us a total of at most $|\Xi|^2 \times \alpha T$ possible model-threshold pairs, represented as $(m, \tau_m(i))$. 

Now, each of these tuples is evaluated on the validation data, with usages and accuracies computed. Again, the pair of models and thresholds with the smallest usage that exceeds the target accuracy on the validation set is selected. 

\paragraph{Alternating Minimisation \emph{and} Sum Relaxation \emph{and} Selective Net} Each $m \in \mathcal{M}$ is a pair $(\gamma, \pi)$, where the former is the gate. Again, on the training data, the value taken by $\gamma$ on each training point is recorded. This gives all the thresholds that may be selected for the gating function. Now, each $m$ and corresponding choice of threshold may be evaluated on the validation set, and we select the ones which match the accuracy requriement and show the lowest usage.

\subsection{Tables Omitted from the Main Text}\label{appx:exp_tables}

\begin{table*}[h]
\centering
    \resizebox{\columnwidth}{!}{%

\begin{tabular}{c|c|ccc|ccc|ccc|ccc|ccc|c}
\multicolumn{1}{c}{\multirow{2}{*}{\textbf{\begin{tabular}[c]{@{}c@{}} Task\end{tabular}}}} &\multicolumn{1}{c}{\multirow{2}{*}{\textbf{\begin{tabular}[c]{@{}c@{}}Target \\ Acc.\end{tabular}}}} & \multicolumn{3}{c}{\textbf{Bracketing}} & \multicolumn{3}{c}{\textbf{Local Thr.}}                              & \multicolumn{3}{c}{\textbf{Alt. Min.}} & \multicolumn{3}{c}{\textbf{Sum relax.}} & \multicolumn{3}{c}{\textbf{Sel. Net.}}& \multicolumn{1}{c}{\multirow{2}{*}{\textbf{Gain}}} \\
         \multicolumn{1}{c}{} & \multicolumn{1}{c}{\textbf{}} & \multicolumn{1}{c}{\textbf{Acc.}} & \multicolumn{1}{c}{\textbf{Usg.}} & \multicolumn{1}{c}{\textbf{ROL}} & \multicolumn{1}{c}{\textbf{Acc.}} & \multicolumn{1}{c}{\textbf{Usg.}} & \multicolumn{1}{c}{\textbf{ROL}} & \multicolumn{1}{c}{\textbf{Acc.}} & \multicolumn{1}{c}{\textbf{Usg.}} & \multicolumn{1}{c}{\textbf{ROL}} & \multicolumn{1}{c}{\textbf{Acc.}} & \multicolumn{1}{c}{\textbf{Usg.}} & \multicolumn{1}{c}{\textbf{ROL}} & \multicolumn{1}{c}{\textbf{Acc.}} & \multicolumn{1}{c}{\textbf{Usg.}} & \multicolumn{1}{c}{\textbf{ROL}}  & \multicolumn{1}{c}{}                               \\
                                    \hline\hline
\multirow{3}{*}{ \textbf{\begin{tabular}[c]{@{}c@{}} MNIST\\Odd/Even\end{tabular}} } 
& 0.995 & 0.994 & 0.457 & 2.19 & 0.995 & 0.653 & 1.53 & 0.991 & 0.830 & 1.20 & 0.997 & 0.785 & 1.27 & 0.996 & 0.658 & 1.52 & 1.431$\times$ \\
& 0.990 & 0.990 & 0.387 & 2.58 & 0.991 & 0.515 & 1.94 & 0.985 & 0.740 & 1.35 & 0.992 & 0.651 & 1.54 & 0.992 & 0.544 & 1.84 & 1.332$\times$ \\
& 0.980 & 0.982 & 0.299 & 3.35 & 0.983 & 0.358 & 2.79 & 0.974 & 0.604 & 1.66 & 0.992 & 0.651 & 1.54 & 0.985 & 0.423 & 2.37 & 1.199$\times$ \\ \hline
\multirow{3}{*}{ \textbf{\begin{tabular}[c]{@{}c@{}} CIFAR\\Random Pair\end{tabular}} } 
& 0.995 & 0.991 & 0.363 & 4.01 & 0.996 & 0.510 & 2.25 & 0.991 & 0.854 & 1.19 & 0.997 & 0.620 & 2.07 & 0.992 & 0.436 & 3.04 &  1.280$\times$\\
& 0.990 & 0.986 & 0.294 & 5.66 & 0.991 & 0.399 & 3.41 & 0.986 & 0.754 & 1.40 & 0.994 & 0.488 & 3.31 & 0.987 & 0.347 & 4.30 &  1.265$\times$\\
& 0.980 & 0.975 & 0.214 & 9.97 & 0.983 & 0.276 & 6.38 & 0.975 & 0.611 & 1.87 & 0.986 & 0.345 & 5.81 & 0.977 & 0.257 & 11.67 &  1.195$\times$\\
\end{tabular}

} \caption{\small Performances on BL tasks studied. This table repeats the entries of Table \ref{table:usages}, with the addition of a column indicating the test accuracy attained by the models. Note that these fluctuate in the range -0.05 to +0.02 of the target, as can be expected from any selection method. \label{table:usages_with_acc} }\vspace{-10pt}
\end{table*}

\begin{table*}[h]
\centering

\begin{tabular}{c|c|c|c}

\textbf{Class Pair} & \textbf{Bracketing} & \textbf{Sel. Net.} & \textbf{Gain}\\\hline\hline

\textbf{0 - 3} & 0.304 & 0.364 & 1.199$\times$\\
\textbf{6 - 4} & 0.452 & 0.526 & 1.164$\times$\\
\textbf{5 - 2} & 0.616 & 0.631 & 1.026$\times$\\
\textbf{6 - 1} & 0.095 & 0.122 & 1.296$\times$\\
\textbf{9 - 3} & 0.220 & 0.211 & 0.961$\times$\\
\textbf{8 - 1} & 0.235 & 0.381 & 1.619$\times$\\
\textbf{7 - 4} & 0.615 & 0.646 & 1.050$\times$\\
\textbf{8 - 7} & 0.059 & 0.091 & 1.538$\times$\\
\textbf{4 - 0} & 0.195 & 0.315 & 1.620$\times$\\
\textbf{6 - 7} & 0.152 & 0.179 & 1.177$\times$\\

\end{tabular}

\caption{\small Usages and relative gain for bracketing and selective net \cite{geifman2019selectivenet} methods at $99\%$ target accuracy for 10 CIFAR random pairs. These two methods uniformly have the lowest usages, and hence the others are omitted. All models achieve test accuracy in the range 98.1-99.3\% test accuracy. Notice that the gains have a large variance, but with a skew towards entries greater than $1$. \label{cifar_binary}}\vspace{-10pt}
\end{table*}

\end{appendix}

\end{document}